\documentclass[11pt]{article}
\usepackage[letterpaper,margin=0.8in]{geometry}
\usepackage{microtype}
\usepackage{bbm,appendix}
\usepackage{xcolor}
\usepackage{cite}
\usepackage{amsmath,amssymb,amsfonts,amsthm}
\usepackage{graphicx}
\usepackage{algorithm,algorithmic}
\usepackage{hyperref}
\hypersetup{hidelinks,
  pdftitle={Data-Driven Exploration for a Class of Continuous-Time Indefinite Linear-Quadratic Reinforcement Learning Problems},
  pdfauthor={Yilie Huang and Xun Yu Zhou},
  pdfsubject={IEEE Transactions on Automatic Control, manuscript 25-3373, final accepted version}
}
\usepackage{textcomp}
\def\BibTeX{{\rm B\kern-.05em{\sc i\kern-.025em b}\kern-.08em
		T\kern-.1667em\lower.7ex\hbox{E}\kern-.125emX}}

\usepackage{enumitem}
\usepackage{tikz} 
\usepackage{subcaption}
\usepackage{booktabs}

\newtheorem{theorem}{Theorem}
\newtheorem{lemma}{Lemma}

\newtheorem{remark}{Remark}

\newcommand{\dd}{\mathrm{d}}

\newcommand{\E}{\mathbb{E}}
\newcommand{\p}{\mathbb{P}}

\newcommand{\g}{\mathcal{G}}
\newcommand{\f}{\mathcal{F}}

\usetikzlibrary{fit,positioning}

\newcommand{\bm}[1]{\boldsymbol{#1}}
\newcommand{\change}[1]{#1}

\newcommand{\Mj}{(\sum_{j=1}^{m} D_j D_j^\top)^{-1}}

\newcommand{\Sj}{\sum_{j=1}^{m}}

\begin{document}
\raggedbottom
	
	\title{Data-Driven Exploration for a Class of Continuous-Time Indefinite
Linear--Quadratic Reinforcement Learning Problems}
	
	\author{Yilie Huang\thanks{Yilie Huang: Department of Applied Mathematics, The Hong Kong Polytechnic University, Hung Hom, Kowloon, Hong Kong (yilie.huang@polyu.edu.hk; tel. \mbox{+852 2766 4010}; fax \mbox{+852 2362 9045}).},
  Xun Yu Zhou\thanks{Xun Yu Zhou (corresponding author): Department of Industrial Engineering and Operations Research \& Data Science Institute, Columbia University, 500 West 120th Street, New York, NY 10027, USA (xz2574@columbia.edu; tel. \mbox{+1 212 854 5237}; fax \mbox{+1 212 854 8103}).}}

\date{}
\maketitle
\makeatletter
\begingroup
  \long\def\@makefntext#1{\noindent\raggedright#1}
  \let\@footnotetext\H@@footnotetext
  \footnotetext{Initially submitted on 29 June 2025. The authors were supported by the Nie Center for Intelligent Asset Management at Columbia University. Their work is also part of a Columbia--CityU/HK collaborative project that is supported by the InnoHK Initiative, the Government of the HKSAR, and the AIFT Lab. Yilie Huang was also supported by a start-up fund from The Hong Kong Polytechnic University.}
  \footnotetext{\textbf{Publication.} \textit{IEEE Transactions on Automatic Control}, Early Access, pp.\ 1--16, 16 September 2026. \href{https://doi.org/10.1109/TAC.2026.3734857}{DOI: 10.1109/TAC.2026.3734857}. Author-accepted manuscript.\protect\newline \copyright\ 2026 IEEE. Personal use of this material is permitted. Permission from IEEE must be obtained for all other uses, in any current or future media, including reprinting/republishing this material for advertising or promotional purposes, creating new collective works, for resale or redistribution to servers or lists, or reuse of any copyrighted component of this work in other works.}
\endgroup
\makeatother

	\begin{abstract}
		We study reinforcement learning (RL) for a  class of continuous-time stochastic linear--quadratic (LQ) control problems, where volatilities depend on both states and controls while states are scalar-valued and there are no running control rewards, the last feature also leading to the so-called \textit{indefinite} LQ controls.  We propose a model-free, data-driven exploration mechanism that adaptively adjusts entropy regularization by the critic and policy variance by the actor. Unlike the constant or deterministic exploration schedules employed in the existing literature, which require extensive tuning for implementations and ignore learning progress during iterations, our adaptive exploratory approach boosts learning efficiency with minimal tuning. Despite its flexibility, our method achieves a sublinear regret bound that matches the best-known model-free results for this class of LQ problems, which were previously derived only with fixed exploration schedules. Numerical experiments demonstrate that adaptive explorations accelerate convergence and improve regret performance compared to the non-adaptive model-free and model-based counterparts.
	\end{abstract}
	
	\par\smallskip
	
	\section{Introduction}
	\label{sec_intro}

    Linear--quadratic (LQ) control is a cornerstone of optimal control theory, widely applied in engineering, economics, and robotics among many others due to its analytical tractability and practical relevance. Its theory has been extensively developed in the classical model-based setting where all the LQ coefficients are known and given \cite{anderson2007optimal,YZbook}.
    Assuming the objective is to {\it minimize} the cost functional, a key assumption in the LQ theory is that the control cost weighting matrix must be positive definite. Intuitively, this is because if increasing the level of control is costless or even beneficial, then the optimal control would just be infinitely large leading to an ill-posed problem.
    \cite{chen1998stochastic} finds, however, that certain stochastic LQ control problems with indefinite cost matrices---including those associated with control---can still be well-posed if the diffusion coefficients depend on the control variable. This seemingly surprising result actually has a simple explanation: while benefiting from an indefinite control cost, a larger control is disadvantageous in terms of increasing the level of uncertainty because the volatility term depends on control multiplicatively. Therefore, one needs to carefully choose the optimal level of control, rendering a well-posed optimization problem. \cite{chen1998stochastic} coins the term ``indefinite stochastic LQ control" for such a problem, and establishes conditions for well-posedness and solvability via newly introduced  generalized Riccati equations. Subsequent works extend the theory in several directions, including infinite-time horizons \cite{ait2000well} and  connections with linear matrix inequalities \cite{rami2000linear}, asymptotic properties of the Riccati solutions \cite{rami2001solvability}, and semidefinite programming  for numerical computations \cite{yao2001primal}. Research on indefinite LQ controls and the related generalized Riccati equations has been active to this day; see, e.g., \cite{li2022stabilization,gashi2023optimal,wang2024indefinite,wu2025stochastic}.

    All the research on indefinite LQ controls, and indeed most studies on stochastic controls,
     are undertaken in a model-based paradigm where the system coefficients and objective/cost functionals are fully known.
    In real-world applications, however, complete knowledge of model parameters is rarely available. Many environments exhibit LQ-like structure yet key parameters in dynamics and objectives may be partially known or entirely unknown. While one can use historical data to estimate those model parameters  as commonly done in the so-called ``plug-in" approach, estimating some of the parameters to a required accuracy may not be possible. For instance, it is statistically impossible to estimate a stock return rate for a very small period -- a phenomenon termed ``mean-blur"  \cite{merton1980estimating, luenberger1998investment}. On the other hand, objective function values  can be highly sensitive to estimated model parameters, rendering inferior performances or even complete irrelevance of model-based solutions \cite{rustem2009algorithms}.

    Given the limitations of model-based controls, particularly the challenges associated with unknown model coefficients, reinforcement learning (RL) rises to provide a promising remedy. Simply put, RL is stochastic control with {\it unknown} model parameters. Yet it never attempts to estimate any model coefficients, and solves the problem in a fundamentally different way compared to the classical, model-based stochastic control. Within RL, LQ control has also served as a key benchmark for studying learning in controlled systems. For deterministic LQ problems, convergence guarantees for learning-based controllers have been obtained; see, for example, \cite{bradtke1992reinforcement,fazel2018global}. For systems with identically and independently distributed noises, regret and finite-sample guarantees for learning-based controllers have been developed, including policy-gradient type methods; see, for example, \cite{abbasi2011regret,abeille2018improved,hambly2021policy,wang2021global,zhou2023single}. These results, however, are all established in discrete time,  typically under the assumption  that the control does not directly affect the volatility level of the state dynamics, in contrast to the continuous-time, state- and control-dependent diffusion setting considered in this paper. Moreover, their exploration mechanisms are tailored to discrete-time models and do not directly extend to continuous-time diffusion dynamics.

    A centerpiece of RL, including the works mentioned above, is exploration, with which the agent interacts with the unknown environment and improves her policies. The simplest exploration strategy is perhaps the so-called $\epsilon$-greedy for bandits problem \cite{sutton2011reinforcement}. More sophisticated and efficient exploration strategies include intrinsic motivation methods that reward exploring  novel or unpredictable states \cite{aubret2019survey, schmidhuber2010formal,achiam2017surprise,pathak2017curiosity, burda2018large,burda2018exploration, osband2018randomized}, count-based exploration that encourages the agent to explore less frequently visited states \cite{menard2021fast, tang2017exploration}, go-explore that explicitly stores and revisits promising past states to enhance sample efficiency \cite{ecoffet2019go, ecoffet2021first}, imitation-based exploration which leverages expert demonstrations \cite{vecerik2017leveraging}, and safe exploration methods that rely on predefined safety constraints or risk models \cite{garcia2015comprehensive, saunders2017trial}. While these methods are shown to be successful in their respective contexts, they are all for discrete-time problems and seem  difficult to be extended to the continuous-time counterparts. However, most real-life applications are inherently continuous-time with continuous state spaces and possibly continuous control spaces (e.g., autonomous driving, robot navigation, and video game playing). On the other hand,  transforming a continuous-time problem into a discrete-time one upfront by discretization has shown to be problematic when the time step size becomes very small \cite{munos2006policy,tallec2019making,park2021time}.

	
	Thus, there is a pressing need for developing exploration strategies tailored to continuous-time RL, particularly and foremost  in the LQ setting, for achieving both stability and learning efficiency. Recent works, such as \cite{huang2024sublinear, szpruch2024optimal}, have employed constant or deterministic exploration schedules to regulate exploration parameters and proved to  achieve sublinear regret bounds. However, these approaches come with notable drawbacks. They require extensive manual tuning of the exploration hyperparameters, which considerably increases computational costs, yet these tuning efforts are typically not reflected in theoretical analyses and results including those concerning  regret bounds. Moreover, pre-determined exploration strategies lack adaptability by nature, as they do not react to the incoming data nor to the current learning states, typically resulting in slower convergence and increased variance from excessive exploration or premature convergence to suboptimal policies from insufficient exploration.
	
	This paper presents  a companion research of \cite{huang2024sublinear}. We continue to study the same class of stochastic LQ problems in \cite{huang2024sublinear}, which is also of an indefinite LQ type because control is absent in the  objective functional. We develop a data-driven adaptive exploration framework for both critic (the value functions) and actor (the control policies). Specifically,  we propose an adaptive exploration strategy that dynamically adjusts the entropy regularization parameter for the critic and the stochastic policy variance for the actor based on the agent's learning progress. Moreover, we provide theoretical guarantees by proving the almost sure convergence of both the adaptive exploration and policy parameters, along with a sublinear regret bound of $O(N^{\frac{3}{4}})$. This bound matches the best-known model-free result from \cite{huang2024sublinear}, which assumes a nonzero initial state and uses a fixed/deterministic exploration schedule, whereas we remove this assumption and learn entirely from data. Finally, we validate our theoretical results through numerical experiments, demonstrating a faster convergence and improved regret bounds compared with a model-based benchmark with fixed exploration. For a comparison with the previous paper \cite{huang2024sublinear}, we design experiments with both fixed and randomly generated model parameters. While theoretically the regret bounds in the two papers are of the same order, the numerical experiments show that our adaptive exploration strategy consistently outperforms and achieves lower regrets.

    A distinctive advantage of the RL approach over the traditional model-based one for indefinite LQ control is that the latter relies heavily on the generalized (or indefinite) Riccati equation, which is a highly complex and unconventional Riccati equation. In particular, it involves a positive definiteness constraint and may be singular due to the indefiniteness of the control weighting cost \cite{chen1998stochastic}. A central task for solving indefinite LQ control problems has been to solve this equation analytically or numerically (e.g. \cite{rami2000linear,rami2001solvability,yao2001primal,hu2003indefinite,qian2013existence}), which remains open for the most general case in terms of its well-posedness. By contrast, RL does not need to involve this equation at all, in the same way that RL gets around the HJB PDE for a general stochastic control problem.

	The remainder of this paper is organized as follows. Section \ref{sec_formulation_and_preliminaries} formulates the problem. Section \ref{sec_exploration} discusses the limitations of deterministic exploration strategies and presents our data-driven adaptive exploration mechanism. Section \ref{sec_rl_algo} describes the continuous-time LQ-RL algorithm with adaptive exploration. Section \ref{sec_result_algo} establishes theoretical guarantees, proving convergence properties and regret bounds. Section \ref{sec_experiments} presents numerical experiments that validate the effectiveness of our approach and compare it against model-free and model-based benchmarks with deterministic exploration schedules. Finally, Section \ref{sec_conclusion} concludes.

	\section{Problem Formulation and Preliminaries}
	\label{sec_formulation_and_preliminaries}
	

	\subsection{Stochastic LQ Control: Classical Setting}
	
	We consider the same stochastic LQ control problem  in \cite{huang2024sublinear}. The one-dimensional state process \( x^u = \{ x^u(t) \in \mathbb{R} : 0 \leq t \leq T \} \) evolves under the multi-dimensional control process \( u = \{ u(t) \in \mathbb{R}^l : 0 \leq t \leq T \} \) according to the stochastic differential equation (SDE):
	\begin{equation}
		\label{eq_classical_dynamics}
		\begin{aligned}
			\mathrm{d}x^u(t) =& (A x^u(t) + B^\top u(t)) \mathrm{d}t \\
			&+ \sum_{j=1}^{m}(C_j x^u(t) + D_j^\top u(t)) \mathrm{d}W^{(j)}(t),
		\end{aligned}
	\end{equation}
	where \( A \) and \( C_j \) are scalars, \( B \) and \( D_j \) are \( l \times 1 \) vectors, and the initial condition is \( x^u(0) = x_0 \). The noise process
	\(W = \{ (W^{(1)}(t), \dots, W^{(m)}(t))^\top \in \mathbb{R}^m : 0 \leq t \leq T \}\)
	is an \( m \)-dimensional standard Brownian motion. We assume \( \sum_{j=1}^{m} D_j D_j^\top > 0 \) to ensure well-posedness of the problem soon to be formulated.
	
	The objective is to find a control process \( u \) that maximizes the expected quadratic reward:
	\begin{equation}
		\label{eq_classical_lq}
		\max_{u} \mathbb{E} \left[ \int_0^T -\frac{1}{2} Q x^u(t)^2 \mathrm{d}t - \frac{1}{2} H x^u(T)^2 \right],
	\end{equation}
	where \( Q \geq 0 \) and \( H \geq 0 \) are scalar weighting parameters. Note that we {\it maximize} the reward functional (with negative signs), instead of minimizing a cost functional as in the usual LQ literature, to be consistent with the setting of \cite{huang2024sublinear}. Because there is no penalty  on control, the problem is an instance of the indefinite LQ control \cite{chen1998stochastic}. In the classical model-based setting, such an indefinite problem requires a careful analysis of the associated {\it indefinite} Riccati equation (e.g., \cite{chen1998stochastic}). In our model-free continuous-time RL framework, we do not solve the Riccati equation, and develop the same theory and algorithm whether the problem is definite or indefinite.

    We are able to treat only the case of one-dimensional state processes because the method in this paper crucially depends on the feature that optimal control turns out to be time-{\it invariant}. Specifically, the optimal feedback control in general takes the form
    \( u_{CL}(t,x) = -\big(\sum_{j=1}^m D_j^\top P(t) D_j\big)^{-1} \big(B^\top P(t) + \sum_{j=1}^m D_j^\top P(t) C_j\big)x \),
    where \( P \) solves the associated matrix Riccati ODE; see \cite[Chapter~6]{YZbook}.
    When $x$ and therefore $P$ are multi-dimensional, this control is time {variant} and our method fails. However, when $x$ is one-dimensional, the $P(t)$'s in the above expression cancel out and the control becomes time independent. The case of multi-dimensional state processes remains an outstanding open question.

	\subsection{Stochastic LQ Control: RL Setting}
	\label{sec_rl_formulation}
	
	We now consider the  {\it model-free}, RL version of the problem just formulated, following \cite{wang2020reinforcement}, without assuming a complete knowledge of the model parameters or relying on parameter estimation. 
	
In this case, there is no Riccati equation to solve (because its coefficients are unknown) and no explicit solutions to compute as in \cite[Chapter 6]{YZbook}. Instead, the RL agent randomizes controls and considers distribution-valued control processes of the form \(\pi = \{\pi(\cdot,t) \in \mathcal{P}(\mathbb{R}^l) : 0 \leq t \leq T\}\), where \(\mathcal{P}(\mathbb{R}^l)\) is the set of all probability density functions over \(\mathbb{R}^l\). The agent employs the control \(u=\{u(t) \in \mathbb{R}^l : 0 \leq t \leq T\}\), where $u(t)\in \mathbb{R}^l$ is sampled from $\pi(\cdot,t)$ at each $t$, to control the original state dynamic.

According to \cite{wang2020reinforcement}, the system dynamic under a randomized control \(\pi\) follows the SDE:
	\begin{equation}
		\mathrm{d}x^\pi(t) = \widetilde{b}(x^\pi(t),\pi(\cdot,t)) \mathrm{d}t + \Sj \widetilde{\sigma}_j(x^\pi(t),\pi(\cdot,t)) \mathrm{d}W^{(j)}(t),
		\label{rl_dynamics}
	\end{equation}
	where
	\begin{equation}
		\widetilde{b}(x,\pi) := A x + B^\top \int_{\mathbb{R}^l} u \pi(u) \mathrm{d}u,
		\label{b_tilde}
	\end{equation}
	\begin{equation}
		\widetilde{\sigma}_j(x,\pi) := \sqrt{\int_{\mathbb{R}^l} (C_j x + D_j^\top u)^2 \pi(u) \mathrm{d}u}. 
		\label{sigma_tilde}
	\end{equation}
	
	To encourage exploration, an entropy term is added to the objective function. 
The entropy-regularized value function associated with a given policy \(\pi\) is defined as
	\begin{equation}
		\begin{aligned}
			\label{eq_value_function_pi}
			J(t, x; \pi) = \mathbb{E} \biggl[ &\int_t^T \left(-\frac{1}{2}Q(x^{\pi}(s))^2 + \gamma p^{\pi}(s) \right) \mathrm{d}s \\
			&- \frac{1}{2}H(x^{\pi}(T))^2 \Big| x^{\pi}(t) = x \biggl],
		\end{aligned}
	\end{equation}
	where \(p^{\pi}(t) = -\int_{\mathbb{R}^l} \pi(t,u) \log \pi(t,u) \mathrm{d}u\) denotes the differential entropy of \(\pi\), and \(\gamma \geq 0\) is the so-called temperature parameter, which represents the trade-off between exploration and exploitation. The optimal value function is then given by
	\(
	V(t,x) = \max_{\pi} J(t, x; \pi).
	\)
	
	Applying standard stochastic LQ control techniques as in \cite[Chapter~6]{YZbook}, \cite[Section~2.2]{huang2024sublinear} derives explicitly the optimal value function and control for the same problem studied in this paper, leading directly to
	\begin{equation}
		\label{eq_rl_value_function}
		V(t, x) = -\frac{1}{2} k_1(t) x^2 + k_3(t),
	\end{equation}
	\begin{equation}
		\label{eq_rl_policy}
		\pi(u \mid t, x) = \mathcal{N}\left(u \Big | \bar{\mu}x, \Sigma \right),
	\end{equation}
	where  \(\mathcal{N}(\cdot | \bar{\mu}x, \Sigma)\) represents a multivariate Gaussian distribution with mean \(\bar{\mu}x\) and covariance matrix \(\Sigma\) with \(\bar{\mu}=-\left(\sum_{j=1}^{m} D_j D_j^\top\right)^{-1} \left(B + \sum_{j=1}^{m} C_j D_j\right)\) and \(\Sigma=\frac{\gamma}{k_1(t)} \Mj\). The functions \(k_1\) and \(k_3\) are determined by ordinary differential equations:
	
	\begin{equation}
		\label{eq_ode_k}
		\begin{aligned}
			k_1^\prime(t)=&-\biggl[2A+2B^\top \bar{\mu} + \Sj \biggl( D_j^\top \bar{\mu}\bar{\mu}^\top D_j+2C_jD_j^\top\bar{\mu} \\
			&+ C_j^2 \biggr)\biggr] k_1(t) - Q, \qquad k_1(T)=H,\\
			k_3^\prime(t) =& \frac{k_1(t)}{2} \Sj D_j^\top \Sigma D_j - \frac{\gamma}{2} \log \biggl( (2\pi e)^l \det(\Sigma) \biggl), \\
			&\qquad k_3(T)=0.
		\end{aligned}
	\end{equation}
	
	It is clear from \eqref{eq_ode_k} that \(k_1(t)>0\) is independent of \(\gamma\), and \(k_1\), \(k_3\) along with their derivatives are uniformly bounded over \([0, T]\).
    \footnote{
    Write
    \(c:=2A+2B^\top\bar{\mu}+\sum_{j=1}^{m}\bigl(D_j^\top\bar{\mu}\bar{\mu}^\top D_j+2C_jD_j^\top\bar{\mu}+C_j^2\bigr)\). Then the explicit solution of $k_1$ is
        \(k_1(t)=-Q/c+(H+Q/c)\exp\bigl(-c(t-T)\bigr)\) if \(c\neq 0\) and \(k_1(t)=H-Q(t-T)\) if \(c=0\).}
	
We stress that the above are theoretical results that cannot be used to compute the optimal solutions because none of the parameters in the expressions is known; yet they offer valuable \textit{structural} insights for developing the  LQ-RL algorithms. Specifically, the optimal value function (critic) is quadratic in the state variable \(x\), while the optimal (randomized) feedback control policy (actor) follows a Gaussian distribution, with its mean exhibiting a linear relationship with \(x\). This intrinsic structure will play a pivotal role in reducing the complexity of function parameterization/approximation in our subsequent learning procedure.

	\subsection{Exploration in LQ-RL: Actor and Critic Perspectives}
	
	In our RL formulation above, both the actor and the critic engage in controls of exploration. The actor does this  through the variance in the Gaussian policy \eqref{eq_rl_policy}, which represents the level of additional randomness injected into the interactions with the environment.  The critic, on the other hand, regulates exploration via the entropy-regularized objective \eqref{eq_value_function_pi}, where the temperature parameter \(\gamma\) determines, if indirectly,  the extent of exploration. 
	Managing exploration properly is both important and challenging: excessive exploration can cause instability and hinder convergence, and insufficient exploration may prevent effective policy improvement. 
The next section introduces a data-driven adaptive exploration strategy designed to address these challenges and optimize policy learning.

	\section{Data-driven Exploration Strategy for LQ-RL}
	\label{sec_exploration}
	
	This section presents the core contribution of this work: data-driven explorations by both  the actor and the critic. 
	
	\subsection{Limitations of Deterministic Exploration Strategies}
	\label{sec_limitations}
	
	A critical challenge in RL is to design exploration strategies that are both effective and efficient. In the setting of continuous-time LQ, existing methods \cite{huang2024sublinear, szpruch2024optimal} employ non-adaptive  strategies by setting exploration levels of both actors and critics  either as fixed constants or some deterministic schedules of time. While these strategies are proved to have theoretical guarantees, they exhibit several essential limitations in learning efficiency and practical implementations.
	
	One major limitation of deterministic explorations lies in their arbitrary nature. Whether fixed as a constant or a predefined sequence, such a schedule often requires extensive manual tuning in actual implementations. The tuning process in turn may introduce  significant computational and time costs for learning; yet  these costs are rarely incorporated into theoretical analyses, creating a disconnect between theoretical guarantees and practical performance.
	
	Another drawback of deterministic explorations is their lack of adaptability, as they follow fixed exploration trajectories regardless of the current states of the iterates, often leading to both excessive and insufficient explorations. While maintaining a constant exploration level is clearly ineffective, a deterministic schedule also suffers from a fundamental limitation: it adjusts the exploration parameter in a {\it predetermined} direction. Consequently, when the current exploration level is either too high or too low, the fixed schedule not only fails to correct this but may indeed exacerbate the situation by moving in the wrong direction. We will illustrate this issue with the experiments presented in Section \ref{sec_exploration_experiment}.

	To address these limitations, we propose an endogenous, data-driven approach that adaptively updates both the actor and critic exploration parameters, which will be shown in the subsequent sections.

	\subsection{Adaptive Critic Exploration}
	\label{sec_critic_exploration}
	
	\paragraph{Critic Parameterization} Motivated by the form of the optimal value function in \eqref{eq_rl_value_function}, we parameterize the value function with learnable parameters $\bm{\theta} \in \mathbb{R}^d$ and temperature parameter \(\gamma \in \mathbb{R}\):
	\begin{equation}
		\label{eq_value_parametrization}
		{J}(t, x; \bm{\theta}, \gamma) = -\frac{1}{2}{k}_1(t; \bm{\theta})x^2 + {k}_3(t; \bm{\theta}, \gamma),
	\end{equation}
	where both \(k_1\) and \(k_3\) and their derivatives are continuous. Both functions can be neural nets; e.g. \( k_1(\cdot; \bm{\theta}) \) can  be, for example, a positive neural network with a sigmoid activation.	
	Furthermore, these functions are constructed so that there are appropriate positive constants $c_1,c_2,c_3$ such that
	\begin{equation}
		\label{eq_k_bounds}
		\frac{1}{c_2} \leq {k}_1(t; \bm{\theta}) \leq {c}_2,  \big|{k}_1^\prime(t; \bm{\theta})\big| \leq {c}_1, \big|{k}_3^\prime(t; \bm{\theta}, \gamma)\big| \leq {c}_3,
	\end{equation}
	for all \(0 \leq t \leq T\), \(|\bm{\theta}| \leq c^{(\bm\theta)}\), and \(|\gamma| \leq c^{(\gamma)}\), where \(c^{(\bm\theta)}>0\) and \(c^{(\gamma)}>0\) are (fixed) hyperparameters. The values of \( c_1, c_2, c_3 \) are determined by the specific parametric forms of \( k_1(\cdot; \bm{\theta}) \) and \( k_3(\cdot; \bm{\theta}, \gamma) \), as well as \(c^{(\bm\theta)}\) and \(c^{(\gamma)}\).
	\footnote{As will be shown in the subsequent sections, the values of $c_1,c_2,c_3,c^{(\bm\theta)},c^{(\gamma)}$ do not affect the order of regret bound. The key requirement here is the boundedness of \(k_1\), \(k_1^\prime\), and \(k_3^\prime\). Our analysis uses only the bounds in~\eqref{eq_k_bounds}, {\it not} the explicit functional forms of \(k_1\) and \(k_3\).}
	The boundedness assumptions \eqref{eq_k_bounds} follow from the fact that when the model parameters are known, the corresponding functions satisfy the same conditions, as shown in Section \ref{sec_rl_formulation}.
	
	In what follows, we introduce the subscript \(n\) to denote values at the \(n\)-th iteration. For instance, \(\gamma_{n}\) represents the updated value of \(\gamma\) at iteration \(n\), which evolves dynamically as learning progresses.

	\paragraph{Data-Driven Temperature Parameter} The temperature parameter \(\gamma\) plays a crucial role in RL in controlling the weight on the entropy-regularized term in the objective function \eqref{eq_value_function_pi}. It governs the level of exploration by the critic 
    for the purpose of  balancing between exploration and exploitation. A larger  $\gamma_n$ increases the weight on entropy and encourages  more exploration, whereas our choice of $\gamma_n$ (see \eqref{eq_gamma_definition})  makes it gradually decrease as learning progresses, shifting the emphasis from exploration towards exploitation.
	
	In the related literature, \(\gamma\) has been either taken  as a given, exogenous constant hyperparameter \cite{szpruch2021exploration,huang2024sublinear} or an ad hoc deterministic sequence of \(n\) \cite{szpruch2024optimal}, which in turn requires extensive tuning in implementations. By contrast, we derive a data-driven, adaptive temperature schedule endogenously, guided by  the regret analysis. Specifically, the update of $\gamma_n$ is chosen to ensure convergence (which will be shown in the proof of Theorem~\ref{thm_Gamma_convergence}) and determined by the \change{learned critic parameter} $\bm\theta_n$ and a predefined deterministic sequence $b_n$ as follows
	
	\begin{equation}
		\label{eq_gamma_definition}
		\gamma_{n} = \frac{c_\gamma \int_0^T k_1(t; \bm\theta_n) \dd t}{b_n T}, \quad \text{for } n = 0,1,2,\dots.
	\end{equation}
	Here $c_\gamma>0$ is selected to be sufficiently large so that $\frac{1}{c_\gamma}$ is smaller than the minimum eigenvalue of $\Mj$, and the sequence $\{b_n\}$ is positive and monotonically increasing with $b_n \uparrow \infty$, so that $\gamma_n$ gradually decreases as more exploration has been done. The normalization term $\frac{\int_0^T k_1(t; \bm\theta_n) \dd t}{T}$ \change{couples the temperature update with the current critic scale, which makes $\gamma_n$ depend on the  critic $\theta_n$ in the general parameterization because the latter is learned from trajectory data through policy evaluation. This normalization} is tailored to our analysis; see the proof of Theorem~\ref{thm_Gamma_convergence}.

	Incidentally, it follows from \eqref{eq_k_bounds} that
	\begin{equation}
		\label{eq_gamma_upper}
		\gamma_{n} = \frac{c_\gamma\int_0^T k_1(t; \bm\theta_n) \, \dd t}{b_n T} \leq \frac{c_\gamma c_2}{b_n}\rightarrow 0
	\end{equation}
	as \(n\rightarrow +\infty\). 

	\subsection{Adaptive Actor Exploration}
	\label{sec_actor_exploration}
	
	\paragraph{Actor Parameterization} Motivated by the structure of the optimal policy in \eqref{eq_rl_policy},  we parameterize the actor using learnable parameters \(\bm\phi \in \mathbb{R}^l\) and \(\Gamma \in \mathbb{S}^l_{++}\):
	\begin{equation}
		\label{eq_policy_parametrization}
		{\pi}(u \mid x; \bm{\phi}, \Gamma) = \mathcal{N}(u \mid \bm\phi x, \Gamma).
	\end{equation}
Write the corresponding entropy as \(p^\pi(t) = p(t;\bm{\phi}, \Gamma)\).

	\paragraph{Adaptive Actor Exploration}
	The variance parameter \(\Gamma\) represents the level of exploration controlled by the actor. A higher variance encourages broader exploration, allowing the agent to sample diverse actions, while a lower variance focuses more on exploitation.
    As will be shown in Sections~\ref{sec_result_algo} and~\ref{sec_experiments},  the adaptive iteration derived  gradually reduces $\Gamma_n$ as the learning signal stabilizes, shifting the actor from exploration towards exploitation.
    In the  existing literature \cite{huang2024sublinear,szpruch2024optimal}, \(\Gamma_n\) is chosen to be a deterministic sequence of \(n\). Such a predetermined schedule does not account for the agent's experience and observed data. Introducing an adaptive \(\Gamma_n\) enables a more flexible and efficient exploration strategy.
	
	To achieve this, we apply the continuous-time policy gradient (PG) martingale orthogonality condition  of \cite[Section~3.2, Theorem~5]{jia2022policy}, which specializes  in our case to the following moment condition:
	
	\begin{equation}
		\label{eq_pg_martingale_0}
		\begin{aligned}
			\mathbb{E}& \biggl[ \int_{0}^{T} \biggl\{ \frac{\partial \log \pi}{\partial \Gamma}\left(u(t) \mid x(t); \bm\phi, \Gamma\right) ( \mathrm{d} J(t, x(t); \bm\theta, \gamma) \\
			&- \frac{1}{2} Q x(t)^2 \mathrm{d} t + \gamma p(t; \bm\phi, \Gamma) \mathrm{d} t ) + \gamma \frac{\partial p}{\partial \Gamma}(t; \bm\phi, \Gamma) \mathrm{d} t \biggr\} \biggr] = 0.
		\end{aligned}
	\end{equation}
	
	To enhance numerical efficiency in the resulting stochastic approximation algorithm, we reparametrize \(\Gamma\) as \(\Gamma^{-1}\). The chain rule implies that the derivative of \(\Gamma^{-1}\) with respect to \(\Gamma\) is a deterministic and time-invariant term, which can thus be omitted while preserving the validity of \eqref{eq_pg_martingale_0}. Consequently,
	
	\[
	\begin{aligned}\label{Zt}
		\mathbb{E}& \biggl[ \int_{0}^{T} \biggl\{ \frac{\partial \log \pi}{\partial \Gamma^{-1}}\left(u(t) \mid x(t); \bm\phi, \Gamma\right) ( \mathrm{d} J(t, x(t); \bm\theta, \gamma) \\
		&- \frac{1}{2} Q x(t)^2 \mathrm{d} t + \gamma p(t; \bm\phi, \Gamma) \mathrm{d} t ) + \gamma \frac{\partial p}{\partial \Gamma^{-1}}(t; \bm\phi, \Gamma) \mathrm{d} t \biggr\} \biggr] = 0.
	\end{aligned}
	\]
	
	The corresponding stochastic approximation algorithm then gives rise  to the updating rule for the actor exploration parameter $\Gamma$:
	\begin{equation}
		\label{eq_Gamma_updates_without_projections}
		\Gamma_{n+1} \leftarrow \Gamma_n - a^{(\Gamma)}_n Z_n(T),
	\end{equation}
	where $a^{(\Gamma)}_n$ is the learning rate, and 
	
	\begin{equation}
		\label{eq_z_def}
		\begin{aligned}
			Z_n(s) =& \int_{0}^{s}\biggl\{\frac{\partial \log \pi}{\partial \Gamma^{-1}}\left(u_n(t) \mid t, x_n(t); \bm\phi_n, \Gamma_n\right)\\
			&\biggl[\mathrm{d} J\left(t, x_n(t); \bm\theta_n, \gamma_n \right) - \frac{1}{2}Qx_n(t)^2 \mathrm{d} t \\
			&+ \gamma_n p\left(t; \bm\phi_n, \Gamma_n\right) \mathrm{d} t \biggr] + \gamma_n \frac{\partial p}{\partial \Gamma^{-1}}\left(t; \bm\phi_n,\Gamma_n \right) \mathrm{d} t \biggl\},
		\end{aligned}
	\end{equation}
	where $0 \leq s \leq T$.
	
	In sharp contrast with \cite{huang2024sublinear}, the fact that \(\Gamma_n\) is no longer a deterministic sequence that monotonically decreases to zero introduces significant analytical challenges. In particular, its learning rates must be carefully chosen to establish the convergence and convergence rate of \(\Gamma_n\); these  are done in Section~\ref{sec_convergence_gamma}. These properties are essential for guaranteeing that the algorithm ultimately attains a sublinear regret bound, and their implications for overall performance will be discussed in the subsequent sections.

	\section{A Continuous-Time RL Algorithm with Data-driven Exploration}
	\label{sec_rl_algo}
	
	This section presents the development of continuous-time RL algorithms with data-driven exploration. We discuss the policy evaluation and policy improvement steps, followed by a projection technique, time discretization, and the final pseudocode.

	\subsection{Policy Evaluation}
	\label{sec_policy_evaluation}
	
	To update the critic parameter \(\bm\theta\), we employ policy evaluation (PE), a fundamental component of RL that focuses on learning the value function of a given policy.
	
	Based on the parameterizations of the value function and policy in \eqref{eq_value_parametrization} and \eqref{eq_policy_parametrization} respectively, we update $\bm{\theta}$ by applying the general continuous-time PE algorithm of \cite[Section~4.2, Proposition~4]{jia2021policy}, i.e., specializing the martingale-orthogonality conditions therein to our LQ model, which yields
	
	\begin{equation}
		\begin{aligned}
			\label{eq_theta_update_without_projection}
			\bm{\theta}_{n+1}& \leftarrow \bm\theta_n + a^{(\bm\theta)}_n\int_0^T \frac{\partial J}{\partial \bm\theta}(t, x_n(t); \bm\theta_n,\gamma_n) \biggl[ - \frac{1}{2} Qx_n(t)^2 \mathrm{d}t\\
			&+\mathrm{d}J(t, x_n(t); \bm\theta_n, \gamma_n) + \gamma_n {p}(t; \bm{\phi}_n, \Gamma_n) \mathrm{d}t \biggr],
		\end{aligned}
	\end{equation}
	where \(a^{(\bm\theta)}_n\) is the corresponding learning rate. We also remark that the above PE method applies equally to definite and indefinite LQ problems and requires no additional treatment for the latter.
	

	\subsection{Policy Improvement}
	\label{sec_policy_gradient}
	
	The remaining learnable parameter  is the mean of the stochastic Gaussian policy, \(\bm\phi\). We again apply the continuous-time PG result of \cite[Section~3.2, Theorem~5]{jia2022policy} to update $\bm\phi$ via stochastic approximation:
	
	\begin{equation}
		\label{eq_phi_updates_without_projections}
		\bm\phi \leftarrow \bm\phi + a^{(\bm\phi)}_n Y_n(T),
	\end{equation}
	where $a^{(\bm\phi)}_n$ is the learning rate and
	\begin{equation}
		\label{eq_y_def}
		\begin{aligned}
			Y_n(s) =& \int_{0}^{s}\biggl\{\frac{\partial \log \pi}{\partial \bm\phi}\left(u_n(t) \mid t, x_n(t); \bm\phi_n, \Gamma_n\right) \biggl[- \frac{1}{2}Q x_n(t)^2 \mathrm{d} t \\
			&+ \mathrm{d} J\left(t, x_n(t); \bm\theta_n, \gamma_n \right) + \gamma_n p\left(t; \bm\phi_n, \Gamma_n\right) \mathrm{d} t\biggr] \\
			&+ \gamma_n \frac{\partial p}{\partial \bm\phi}\left(t; \bm\phi_n, \Gamma_n \right) \mathrm{d} t \biggl\},
		\end{aligned}
	\end{equation}
	where $0 \leq s \leq T$.

	\subsection{Projections}
	\label{sec_projections}
	
	The parameter updates for \(\bm\theta\), \(\bm\phi\), and \(\Gamma\) follow stochastic approximation (SA), a foundational method introduced by \cite{robbins1951stochastic} and extensively studied in subsequent works \cite{lai2003stochastic, borkar2009stochastic, chau2014overview}. However, directly applying SA in our setting presents challenges such as extreme state values and unbounded estimation errors, which can destabilize the learning process. To address these issues, we incorporate the projection technique proposed by \cite{andradottir1995stochastic}, ensuring that parameter updates remain in a bounded region at every iteration. 
	
	Define \(\Pi_K(x) := \arg \min_{y \in K} |y-x|^2\),  the projection of a point \(x\) onto a set \(K\). We now introduce the projection sets for all the critic and actor parameters.
	
	First, we define the projection sets for the critic parameter \(\bm\theta\) and the temperature parameter \(\gamma\):
	\begin{equation}
		\label{eq_projection_sets_critic}
		K^{(\bm\theta)}=\left\{ \bm\theta\in \mathbb{R}^d \mid |\bm\theta| \leq c^{(\bm\theta)} \right\}, K^{(\gamma)}=\left\{ \gamma\in \mathbb{R} \mid |\gamma| \leq c^{(\gamma)} \right\},
	\end{equation}
	which are employed to enforce the boundedness conditions in \eqref{eq_k_bounds}. The update rules incorporating projection  for \(\bm\theta\) and \(\gamma\) are modified to
	\begin{equation}
		\label{eq_theta_update}
		\begin{split}
			\bm\theta_{n+1} \leftarrow \Pi_{K^{(\bm\theta)}}\biggl( & \bm\theta_n + a^{(\bm\theta)}_n\int_0^T \frac{\partial J}{\partial \bm\theta}(t,x_n(t); \bm\theta_n, \gamma_n) \\
			& \biggl[\mathrm{d} J\left(t, x_n(t); \bm\theta_n, \gamma_n \right) - \frac{1}{2}Qx_n(t)^2 \dd t \\
			&+ \gamma_n {p}\left(t; \bm\phi_n,\Gamma_n\right) \mathrm{d} t\biggr] \biggl),
		\end{split}
	\end{equation}
	\begin{equation}
		\label{eq_gamma_update}
		\begin{split}
			\gamma_{n+1} \leftarrow \Pi_{K^{(\gamma)}}\biggl( \frac{c_\gamma \int_0^T k_1(t; \bm\theta_n) \dd t}{b_n T} \biggl).
		\end{split}
	\end{equation}

	Next, we define the projection sets for the actor parameters \(\bm\phi\) and \(\Gamma\):
	\begin{equation}
		\label{eq_projection_sets_actor}
		\begin{aligned}
			&K^{(\bm\phi)}_n=\left\{ \bm\phi_n \in \mathbb{R}^l \Big| |\bm\phi_n|\leq c^{(\bm\phi)}_n  \right\}, \\
			&K^{(\Gamma)}_n=\left\{ \Gamma_n \in \mathbb{S}^{l}_{++} \Big| |\Gamma_n| \leq c^{(\Gamma)}_n, \Gamma_n - \frac{1}{b_n}I \in \mathbb{S}^l_{++}  \right\},
		\end{aligned}
	\end{equation}
	where \(\{c^{(\bm\phi)}_n \uparrow \infty\}\), \(\{c^{(\Gamma)}_n \uparrow \infty\}\), and \(\{b_n \uparrow \infty\}\) are positive, monotonically increasing sequences, with \(b_n\) the same deterministic sequence in \eqref{eq_gamma_definition}. The specifics of these sequences are given in Theorem \ref{thm_convergence_and_rate_gamma}. Clearly, the sequences of sets \(K^{(\bm\phi)}_n\) and \(K^{(\Gamma)}_n\) expand over time, eventually covering the entire spaces \(\mathbb{R}^l\) and \(\mathbb{S}^{l}_{++}\), respectively. This ensures that our algorithm remains model-free without needing to have prior knowledge of the model parameters. The update rules with projection for \(\bm\phi\) and \(\Gamma\) are now
	\begin{equation}
		\label{eq_phi_update}
		\begin{split}
			\bm\phi_{n+1} \leftarrow& \Pi_{K^{(\bm\phi)}_{n+1}} \biggl( \bm\phi_n+a^{(\bm\phi)}_n Y_{n}(T) \biggl),
		\end{split}
	\end{equation}
	\begin{equation}
		\label{eq_Gamma_update}
		\begin{split}
			\Gamma_{n+1} \leftarrow& \Pi_{K^{(\Gamma)}_{n+1}} \biggl( \Gamma_n-a^{(\Gamma)}_n Z_{n}(T) \biggl),
		\end{split}
	\end{equation}
	where \(Y_{n}(T)\) and \(Z_{n}(T)\) are respectively determined by \eqref{eq_y_def} and \eqref{eq_z_def}.

	\subsection{Discretization}
	\label{sec_discretization}
	
	While our RL framework is continuous-time in both formulation and analysis, numerical implementation requires discretization at the final stage. 
To facilitate this, we divide the time horizon \([0, T]\) into uniform intervals of size \(\Delta t_n\) at the \(n\)-th iteration, resulting in discretized rules:\footnote{Analysis on discretization errors is covered and discussed in Theorems \ref{thm_Gamma_convergence} - \ref{thm_convergence_and_rate_phi}, and Remark \ref{remark_discretization_error}.}
	
	\begin{equation}
		\label{eq_theta_update_discrete}
		\begin{aligned}
			\bm\theta_{n+1} &\leftarrow \Pi_{K^{(\bm\theta)}}\biggl( \bm\theta_n+a^{(\bm\theta)}_n \sum_{k=0}^{\left\lfloor \frac{T}{\Delta t_n} -1 \right\rfloor} \frac{\partial J}{\partial \bm\theta}(t_k,x_n(t_k); \bm\theta_n, \gamma_n)\\
			&\biggl[ J\left(t_{k+1}, x_n(t_{k+1}); \bm\theta_n, \gamma_n \right) - J\left(t_k, x_n(t_k); \bm\theta_n, \gamma_n \right) \\
			&- \frac{1}{2}Qx_n(t_k)^2 \Delta t_n + \gamma_n {p}\left(t_k; \bm\phi_n, \Gamma_n\right) \Delta t_n\biggl] \biggl),
		\end{aligned}
	\end{equation}
	
	\begin{equation}
		\label{eq_gamma_update_discrete}
		\begin{split}
			\gamma_{n+1} \leftarrow \Pi_{K^{(\gamma)}}\biggl( \frac{c_\gamma}{b_n T} \sum_{k=0}^{\left\lfloor \frac{T}{\Delta t_n} -1 \right\rfloor} k_1(t_k; \bm\theta_n) \Delta t_n \biggl).
		\end{split}
	\end{equation}
Moreover, 	\(Y_n(T)\) and \(Z_n(T)\) in the update rules \eqref{eq_phi_update} and \eqref{eq_Gamma_update} are approximated by
	\begin{equation}
		\label{eq_phi_update_discrete}
		\begin{aligned}
			&\hat{Y}_n(T) = \sum_{k=0}^{\left\lfloor \frac{T}{\Delta t_n} -1 \right\rfloor} \Delta t_n \biggl\{\frac{\partial \log \pi}{\partial \bm\phi}\left(u_n(t_k) \mid t_k, x_n(t_k); \bm\phi_n, \Gamma_n\right)\\
			&\biggl[ (J\left(t_{k+1}, x_n(t_{k+1}); \bm\theta_n, \gamma_n \right)  - J\left(t_k, x_n(t_k); \bm\theta_n, \gamma_n \right)) \frac{1}{\Delta t_n} \\
			&- \frac{1}{2}Qx_n(t_k)^2 + \gamma_n {p}\left(t_k; \bm\phi_n, \Gamma_n\right)\biggl]+\gamma_n \frac{\partial {p}}{\partial \bm\phi}\left(t_k; \bm\phi_n, \Gamma_n \right) \biggl\} ,
		\end{aligned}
	\end{equation}
	
	\begin{equation}
		\label{eq_Gamma_update_discrete}
		\begin{aligned}
			&\hat{Z}_n(T) = \sum_{k=0}^{\left\lfloor \frac{T}{\Delta t_n} -1 \right\rfloor} \Delta t_n \biggl\{\frac{\partial \log \pi}{\partial \Gamma^{-1}}\left(u_n(t_k) \mid t_k, x_n(t_k); \bm\phi_n, \Gamma_n\right)\\
			&\biggl[ (J\left(t_{k+1}, x_n(t_{k+1}); \bm\theta_n, \gamma_n \right) - J\left(t_k, x_n(t_k); \bm\theta_n, \gamma_n \right)) \frac{1}{\Delta t_n} \\
			&- \frac{1}{2}Qx_n(t_k)^2 + \gamma_n {p}\left(t_k; \bm\phi_n,\Gamma_n\right) \biggl]+
			\gamma_n \frac{\partial {p}}{\partial \Gamma^{-1}}\left(t_k; \bm\phi_n, \Gamma_n \right) \biggl\} .
		\end{aligned}
	\end{equation}

	\subsection{LQ-RL Algorithm Featuring Data-Driven Exploration}
	\label{sec_details_algo1}
	
	The preceding analysis gives rise to the following RL algorithm that integrates adaptive exploration:

			
			
			
			

	\begin{algorithm}[H]
		\caption{Data-driven Exploration LQ-RL Algorithm}
		\label{algo_rl_lq}
		\begin{algorithmic}
			
			\STATE {\textbf{Input:}} Initial values of learnable parameters \(\bm\theta_0, \bm\phi_0, \gamma_0, \Gamma_0\)
			
			\FOR{\(n = 0\) to \(N\)}
			\STATE Initialize \(k = 0\), time \(t = t_k = 0\), and state \(x_n(t_k) = x_0\).
			
			\WHILE{\(t < T\)}
			\STATE Sample action \(u_n(t_k) \) from policy by (Eq.~\eqref{eq_policy_parametrization}).
			\STATE Update state \(x_n(t_{k+1})\) by LQ dynamics (Eq.~\eqref{eq_classical_dynamics}).
			\STATE Update time: \(t_{k+1} \gets t_k + \Delta t\), and set \(t \gets t_{k+1}\).
			\ENDWHILE
			\STATE Collect trajectory data: \(\{(t_k, x_n(t_k), u_n(t_k))\}_{k \geq 0}\).
			\STATE Update value function parameters \(\bm\theta_{n+1}\) by (Eq.~\eqref{eq_theta_update_discrete}).
			\STATE Update policy mean parameters \(\bm\phi_{n+1}\) by (Eq.~\eqref{eq_phi_update_discrete}).
			\STATE Update critic exploration parameter \(\gamma_{n+1}\) by (Eq.~\eqref{eq_gamma_update_discrete}).
			\STATE Update actor exploration parameter \(\Gamma_{n+1}\) by (Eq.~\eqref{eq_Gamma_update_discrete}).
			\ENDFOR
			
			\STATE \textbf{Output:} Final parameters \(\bm\theta_N, \bm\phi_N, \gamma_N, \Gamma_N\).
			
		\end{algorithmic}
	\end{algorithm}
	
	Algorithm~1 is on-policy: at iteration \(n\), the target policy \(\pi_n\) is updated based on data generated by \(\pi_n\) itself. Off-policy variants within the same continuous-time actor--critic framework are also possible. For example, in the mean--variance portfolio selection application \cite{huang2024mean} one learns a deterministic policy using data generated by Gaussian exploratory policies. More general off-policy learning theory has been established in \cite{jia2023q} for continuous-time $q$-learning.

	\section{Regret Analysis}
	\label{sec_result_algo}
	
	This section contains the main contribution of this work: establishing a sublinear regret bound for the LQ-RL algorithm with data-driven exploration, presented in Algorithm \ref{algo_rl_lq}. We prove that the algorithm achieves a regret bound of \(O(N^{\frac{3}{4}})\) (up to a logarithmic factor), matching the best-known model-free bound for this class of continuous-time LQ problems as reported in \cite{huang2024sublinear}.
	
	To quickly recap the major differences from  \cite{huang2024sublinear}: 1) the actor exploration parameter \(\Gamma_n\) therein is a deterministic monotonically decreasing sequence of  \(n\) while we apply policy gradient for updating \(\Gamma_n\); 2) the static temperature parameter \(\gamma\) is used in \cite{huang2024sublinear}, whereas in our algorithm \(\gamma\) is adaptively updated; 3) \cite{huang2024sublinear} only considers the case when the initial state is nonzero (\(x_0 \neq 0\)), and our analysis removes this assumption. These differences beg for a substantially different theoretical analysis compared with \cite{huang2024sublinear}. In particular, the stochastic update of \(\Gamma_n\) calls for proving the corresponding  almost-sure convergence and explicit convergence rates (which feed into the regret bound), and  the case of \(x_0=0\) imposes additional constraints on the hyperparameters rendering more technical difficulties in both the convergence and regret analysis.
	
	In the remainder of the paper, we use \(c\) (and its variants) to denote  generic positive constants. These constants depend solely on the model parameters \(A, B, C_j, D_j, Q, H, x_0, T, \gamma, m, l\), and their values may vary from one instance to another.

	\subsection{Convergence Analysis on \texorpdfstring{$\gamma_n$}{γₙ} and \texorpdfstring{$\Gamma_n$}{Γₙ}}
	\label{sec_convergence_gamma}
	
	To start, we explore the almost sure convergence of the critic and actor exploration parameters $\gamma_n$ and $\Gamma_n$, together with the convergence rate of $\Gamma_n$ in terms of the mean-squared error (MSE). Throughout the rest of the paper, for real numbers $a,b$ we write $a \vee b := \max\{a,b\}$ and $a \wedge b := \min\{a,b\}$.
	
	\begin{theorem}
		\label{thm_convergence_and_rate_gamma}
		Consider Algorithm \ref{algo_rl_lq} with fixed positive hyperparameters \(c_1, c_2, c_3\) and a sufficiently large fixed constant \(c_\gamma\). Define the sequences:
		\[
		\begin{aligned}
			&a^{(\Gamma)}_n  = a^{(\bm\phi)}_n  = \frac{\alpha^{\frac{3}{4}}}{(n+\beta)^{\frac{3}{4}}}, \quad
			b_n = 1 \vee \frac{(n+\beta)^{\frac{1}{4}}}{\alpha^{\frac{1}{4}}}, \\
			&c^{(\bm\phi)}_n=1 \vee (\log \log n)^{\frac{1}{6}}, c^{(\Gamma)}_n = 1 \vee \log n, \Delta t_n = T(n+1)^{-\frac{5}{8}},
		\end{aligned}
		\]
		where \(\alpha > 0\) and \(\beta > 0\) are constants. Then, the following results hold:
		
		\begin{enumerate}[label=(\alph*)]
			\item As \(n \to \infty\), \(\gamma_n\) converges almost surely to \(0\), and \(\Gamma_n\) converges almost surely to the zero matrix \(\bm{0}\).
			\item For all \(n\), 
			\(
			\mathbb{E}[|\Gamma_n|^2] \leq c \frac{(\log n)^{p_2} (\log \log n)^{\frac{4}{3}}}{n^{\frac{1}{2}}},
			\)
			where \(c\) and \(p_2\) are positive constants.
		\end{enumerate}
	\end{theorem}
	
	These results guarantee that the adaptive exploration parameters \(\gamma_n\) and \(\Gamma_n\) diminish almost surely, and the convergence rate of \(\Gamma_n\) is essential for the proof of the regret bound. The remainder of this subsection is devoted to the proof of Theorem \ref{thm_convergence_and_rate_gamma}.

	
	Define the conditional expectation of \(Z_n(T)\) given the current parameter iterates as
	\[
	h^{(\Gamma)}(\bm\phi_n, \Gamma_n; \bm\theta_n, \gamma_n) = \mathbb{E}[Z_n(T) \mid \bm\theta_n, \bm{\phi}_n, \gamma_n, \Gamma_n],
	\]
	as well as  the deviation 
	\begin{equation}
    \label{eq_xi_definition}
	\xi^{(\Gamma)}_{n} = Z_n(T) - h^{(\Gamma)}(\bm\phi_n, \Gamma_n; \bm\theta_n, \gamma_n).
	\end{equation}
	The update rule \eqref{eq_Gamma_update} can be rewritten as
	\begin{equation}
		\label{eq_Gamma_update_app}
		\Gamma_{n+1} = \Pi_{K^{(\Gamma)}_{n+1}} \left(\Gamma_n - a^{(\Gamma)}_n[h^{(\Gamma)}(\bm\phi_n, \Gamma_n; \bm\theta_n, \gamma_n) + \xi^{(\Gamma)}_{n}]\right).
	\end{equation}
	
	The proof of  Theorem \ref{thm_convergence_and_rate_gamma} is complex and will be subsequently broken into several steps for the reader's convenience. Here let us highlight the key components of the proof. First, we analyze the variance of the increment \(Z_n(T)\) to  ensure that the stochastic fluctuations along  iterations remain controlled (Section~\ref{subsub_var_bound}). Second, we study the conditional mean, \(h^{(\Gamma)}(\bm\phi_n, \Gamma_n; \bm\theta_n, \gamma_n)\),  and obtain its bound (Section~\ref{subsub_mean_Gamma}). Third, combining these results with a projected stochastic-approximation argument, we establish the almost sure convergence of 
    $\Gamma_n$ (Section~\ref{subsub_convergence_Gamma}). Finally, we derive an iterative inequality for  $\mathbb{E}[|\Gamma_n|^2]$ and prove  the desired convergence rate  by mathematical induction (Section~\ref{subsub_convergence_rate_Gamma}).

    Now we present the proof of  Theorem \ref{thm_convergence_and_rate_gamma}.

	\subsubsection{A Variance Upper Bound}
    \label{subsub_var_bound}
	
	The following result establishes  an upper bound on the variance of the increment \(Z_n(T)\).
	
	\begin{lemma}
		\label{lemma_nosie_Gamma_app}
		There exists a constant \(c > 0\) depending solely on the model parameters such that
		\begin{equation}
			\label{eq_noise_upper2}
			\begin{aligned}
				&\operatorname{Var}\left( Z_n(T) \Big| \bm\theta_n, \bm\phi_n, \gamma_n, \Gamma_n \right)  \\
				\leq& c \left( 1 + |\bm\phi_n|^{8} + |\Gamma_n|^{8} + (\log b_n)^8 + \gamma_n^8 \right) \exp{\{c|\bm\phi_n|^4\}}.
			\end{aligned}
		\end{equation}
	\end{lemma}
	
	\begin{proof} The proof is similar to that of \cite[Lemma B.2]{huang2024sublinear} except that we need to account for the estimates on $\Gamma_n$; so we will be brief here.
	It follows from \eqref{eq_z_def} together with Ito's lemma (applied to \(J\left(t, x_n(t); \bm\theta_n, \gamma_n \right)\)) that
		\begin{equation}
			\label{eq_dz2_app}
			\begin{aligned}
				\dd Z_n(t) \triangleq& Z_n^{(1)}(t) \dd t + \Sj Z_n^{(2, j)}(t) \dd W_n^{(j)}(t),
			\end{aligned}
		\end{equation}
for which we have the following estimates
		\[
		\begin{aligned}
			\E[|Z_n^{(1)}(t)|^2 |& \bm\theta_n, \bm\phi_n, \gamma_n, \Gamma_n, x_n(t)] \leq c (1+|\Gamma_n|^4)(1+x_n(t)^4\\
			&+|\bm\phi_n|^4x_n(t)^4+\gamma_n^2+\gamma_n^2(\log( \det (\Gamma_n)))^2),
		\end{aligned}
		\]
		\[
		\begin{aligned}
			\E&[|Z_n^{(2,j)}(t)|^2 | \bm\theta_n, \bm\phi_n, \gamma_n, \Gamma_n, x_n(t)] \leq c (1+|\Gamma_n|^3)(1+\\
			&x_n(t)^4+|\bm\phi_n|^2x_n(t)^4+|\Gamma_n|^2x_n(t)^4+|\bm\phi_n|^2|\Gamma_n|^2x_n(t)^2).
		\end{aligned}
		\]
		Taking expectations in the above conditional on  $x_n(t)$ and applying \cite[Lemma B.1]{huang2024sublinear},  we get
		\begin{equation}
			\label{eq_Z2_square_app}
			\begin{aligned}
				&\E[|Z_n^{(1)}(t)|^2 + \Sj|Z_n^{(2,j)}(t)|^2 | \bm\theta_n, \bm\phi_n, \gamma_n, \Gamma_n] \\
				\leq& c \left( 1 + |\bm\phi_n|^{8} + |\Gamma_n|^{8} + (\log b_n)^8 + \gamma_n^8 \right) \exp{\{c|\bm\phi_n|^4\}}.
			\end{aligned}
		\end{equation}
This establishes the desired inequality.
	\end{proof}

	\subsubsection{A Mean Increment}
	\label{subsub_mean_Gamma}
	
	We now analyze the mean increment \(h^{(\Gamma)}(\bm\phi_n, \Gamma_n; \bm\theta_n, \gamma_n)\) in the updating rule \eqref{eq_Gamma_update_app}. 
    Integrating and taking conditional expectation in \eqref{eq_dz2_app}, the stochastic integral term vanishes; so only $Z_n^{(1)}$ contributes to the mean increment. Moreover, since the control is sampled from a Gaussian distribution, all the required moments of the control can be computed explicitly, yielding
	\begin{equation}
		\begin{aligned}
			\label{eq_h2_app}
			&h^{(\Gamma)}(\Gamma_n; \bm\theta_n, \gamma_n) \\
			=& \frac{1}{2}\int_0^Tk_1(t; \bm\theta_n) \dd t \Gamma_n \Bigg(\sum_j D_j D_j^\top \Bigg) \Gamma_n - \frac{\gamma_n T}{2}\Gamma_n.
		\end{aligned}
	\end{equation}
	
	Given that \(h^{(\Gamma)}(\Gamma_n; \bm\theta_n, \gamma_n)\) is quadratic in \(\Gamma_n\), we apply \eqref{eq_k_bounds} to establish the bound:
	\begin{equation}
		\label{eq_h2n_upper_app}
		\left| h^{(\Gamma)}(\Gamma_n; \bm\theta_n, \gamma_n) \right| \leq c (1 + |\Gamma_n|^2 + \gamma_n^2).
	\end{equation}
	

	\subsubsection{Almost Sure Convergence of \texorpdfstring{$\gamma_n$}{γₙ} and \texorpdfstring{$\Gamma_n$}{Γₙ}}
	\label{subsub_convergence_Gamma}
	
	We now prove Part (a) of Theorem \ref{thm_convergence_and_rate_gamma}. Indeed we will present and prove a more general result that involves a bias term $\beta^{(\Gamma)}_n = \mathbb{E}[ \xi^{(\Gamma)}_{n} \mid \mathcal{G}_n]$ (where $\mathcal{G}_n$ is defined in Theorem~\ref{thm_Gamma_convergence}), which captures implementation errors such as the discretization error (see Remark \ref{remark_discretization_error}). When $\beta^{(\Gamma)}_n = \mathbf{0}$, the result simplifies to Theorem \ref{thm_convergence_and_rate_gamma}-(a).  Moreover, by the definition of $\xi^{(\Gamma)}_n$ in \eqref{eq_xi_definition} and Lemma~\ref{lemma_nosie_Gamma_app}, the conditional expectation and moment assumptions on $\xi^{(\Gamma)}_n$ stated in Theorem~\ref{thm_Gamma_convergence} are automatically satisfied.
	
	\begin{theorem}
		\label{thm_Gamma_convergence}
		Assume $\xi^{(\Gamma)}_{n}$ satisfies $\E\left[ \xi^{(\Gamma)}_{n} \Big|\g_n\right] = \beta^{(\Gamma)}_n$ and
		\begin{equation}\label{eq_Gamma_noise_assumption}
			\begin{aligned}
				&\E\left[ \left|\xi^{(\Gamma)}_{n} - \beta^{(\Gamma)}_n \right|^2 \Big| \g_n \right]  \\
				\leq& c (1 + |\bm\phi_n|^8 + |\Gamma_n|^8 + (\log b_n)^8 + \gamma_n^8) \exp{\{c|\bm\phi_n|^4\}}, \\
			\end{aligned}
		\end{equation}
		where $c>0$ is a constant independent of $n$, and $\{\g_n\}$ is the filtration generated by $\{\bm\theta_m, \bm\phi_{m}, \gamma_m, \Gamma_m; m=0,1,2,...,n\}$.
		Moreover, assume
		\begin{equation} \label{eq_Gamma_assumption_app}
			\begin{aligned}
				(i)&  0<a^{(\Gamma)}_n\leq1 \text{ } \forall n,  \sum_n a^{(\Gamma)}_n = \infty,  \sum_n a^{(\Gamma)}_n |\beta^{(\Gamma)}_n| < \infty; \\
				(ii)&  c^{(\bm\phi)}_n\uparrow \infty,  c^{(\Gamma)}_n\uparrow \infty, \text{ and for } 0\leq q_1, q_2, q_3 , 0\leq q_4 \leq 4,\\
				&\sum (a^{(\Gamma)}_n)^2 (c^{(\Gamma)}_n)^ {q_1} (c^{(\bm\phi)}_n)^ {q_2} (\log b_n)^{q_3}   e^{c(c^{(\bm\phi)}_n)^{q_4}} < \infty; \\
				(iii)&  b_n\geq1,  b_n\uparrow \infty,  \sum \frac{a^{(\Gamma)}_n}{b_n}=\infty,  \sum |\frac{1}{b_n}-\frac{1}{b_{n+1}}|<\infty.
			\end{aligned}
		\end{equation}
		
		Then, $\gamma_n \xrightarrow{a.s.} 0$ and $\Gamma_n \xrightarrow{a.s.} \bm{0}$.
	\end{theorem}

    The convergence of $\gamma_n$ is obvious so one needs to focus on that of $\Gamma_n$. By viewing its iteration as a projected stochastic-approximation scheme and adapting the almost sure convergence analysis of \cite{andradottir1995stochastic} under relaxed conditions, we obtain the convergence. A full proof is relegated to  Appendix~\ref{appendix:proof_thm2}.

	\begin{remark}
		\label{remark_example_gamma}
		A case satisfying the assumptions in \eqref{eq_Gamma_assumption_app} is when $a^{(\Gamma)}_n =1\wedge\frac{\alpha^{\frac{3}{4}}}{(n+\beta)^{\frac{3}{4}}}$ and $b_n = 1\vee\frac{(n+\beta)^{\frac{1}{4}}}{\alpha^{\frac{1}{4}}}$, where constants $\alpha>0$, $\beta>0$. $c^{(\bm\phi)}_n = 1\vee (\log \log n)^{\frac{1}{6}}, c^{(\Gamma)}_n = 1\vee \log n, \beta^{(\bm\phi)}_n=0$. 
		This is because $\sum \frac{1}{n} = \infty$; $\sum \frac{(\log n)^{p} (\log \log n)^{q}}{n^{r}} <\infty$, for any $p, q  > 0$ and $r > 1$; and $\sum |n^{-\frac{1}{4}} - (n+1)^{-\frac{1}{4}}| < \frac{1}{4} \sum n^{-\frac{5}{4}} < \infty$ (by the mean--value theorem).
				Additionally, it is interesting to note that the assumptions \(\sum \frac{a^{(\Gamma)}_n}{b_n} = \infty\) and \(\sum |\frac{1}{b_n}-\frac{1}{b_{n+1}}|<\infty\) are new compared to \cite{huang2024sublinear}. This is due to the implementation of data-driven exploration.
	\end{remark}

	\subsubsection{Convergence Rate of \texorpdfstring{$\Gamma_n$}{Γₙ}}
	\label{subsub_convergence_rate_Gamma}
	
	We need the following lemmas for the proof of Theorem~\ref{thm_Gamma_rate}, now that \(\Gamma_n\) is updated from data rather than following a deterministic schedule as in \cite{huang2024sublinear}. Specifically, Lemma~\ref{lemma_an_app} is used to deal with the term \(b_n\) in \eqref{eq_inner_lower_bound}, while Lemma~\ref{lemma_gn_bound} provides the order of \(|G_n|^2 = |\Gamma_n^* - \Gamma_{n+1}^*|^2\).
	\begin{lemma}
		\label{lemma_an_app}
		For any $W > 0$ and any $0<q<1$, there exist positive numbers $\alpha > \frac{1}{W}$ and $\beta \geq \max(\frac{1}{W \alpha -1}, W^{\frac{2q-1}{1-q}}\alpha^\frac{q}{1-q})$ such that the sequences $\hat{a}_n = \frac{\alpha}{n + \beta}$ and $a_n=\frac{\alpha^q}{(n+\beta)^q}$ satisfy $\hat{a}_n \leq \hat{a}_{n+1}(1 + W \hat{a}_{n+1})$ and $a_n \leq a_{n+1}(1 + W a_{n+1})$ for all $n \geq 0$.
	\end{lemma}
	
	\begin{proof}
		First,
		\[
		\begin{aligned}
			\hat{a}_n \leq \hat{a}_{n+1}(1 + W \hat{a}_{n+1})\Leftrightarrow n + 1 + \beta \leq W \alpha n + W \alpha \beta,\\
		\end{aligned}
		\]
		the latter being  true for any $n$ when $\alpha > \frac{1}{W}>0$, $\beta \geq \frac{1}{W \alpha -1}>0$.
		
		Next, we have
		\begin{equation}
			\label{eq_lemma_an}
			\begin{aligned}
				&a_n \leq a_{n+1}(1 + W a_{n+1})\\
				\Leftrightarrow &(n+\beta+1)^q - (n+\beta)^q \leq W \alpha^q \biggl( \frac{n+\beta}{n+\beta+1} \biggl)^q.
			\end{aligned}
		\end{equation}
For the latter inequality in \eqref{eq_lemma_an}, notice that the left-hand side decreases in $n$ while the right-hand side increases in $n$. So to show that this inequality is true for all $n$, it is sufficient to show that it is true when $n=0$, which is
		$(\beta+1)^q - \beta^q \leq W\alpha^q \frac{\beta^q}{(\beta+1)^q}$.
		
		Now, for $\beta \geq \frac{1}{W \alpha -1}$, we have $\beta+1 \leq W\alpha\beta$. On the other hand, the mean--value theorem yields $(\beta+1)^q-\beta^q \leq q\beta^{q-1}$. Hence, 
		\[
		\begin{aligned}
			W^{\frac{2q-1}{1-q}}\alpha^\frac{q}{1-q}\leq&\beta \Rightarrow (\beta+1)^q - \beta^q \leq W\alpha^q \frac{\beta^q}{(\beta+1)^q}.\\
		\end{aligned}
		\]
	\end{proof}

	\begin{lemma}
		\label{lemma_gn_bound}
		Let the sequence $\hat{G}_n = n^{-q} - (n+1)^{-q}>0$ for some $q > 0$. Then $\hat{G}_n^2 = O(n^{-2q-2})$.
	\end{lemma}
	
	\begin{proof}
		We have
		\[
		\hat{G}_n^2 = \left[\frac{(n+1)^{q} - n^{q}}{(n+1)^{q} n^{q}}\right]^2 < \left[\frac{(n+1)^{q} - n^{q}}{n^{2q}}\right]^2.
		\]
Applying the mean--value theorem to the function $f(x) = x^q$ and noting that $f'(x) = qx^{q-1}$ is a decreasing function for $x > 0$, we conclude
		
		\[
		\left[\frac{(n+1)^{q} - n^{q}}{n^{2q}}\right]^2 < \frac{(qn^{q-1})^2}{n^{4q}} = q^2 n^{-2q-2}.
		\]
	\end{proof}

	The following theorem specializes to Theorem \ref{thm_convergence_and_rate_gamma}-(b) when the bias term \(\beta^{(\Gamma)}_n = 0\).
	
	\begin{theorem}
		\label{thm_Gamma_rate}
		Under the same setting of Theorem \ref{thm_Gamma_convergence}, if the sequence \(\{\hat{a}_n\}\), defined as \(\{\frac{a^{(\Gamma)}_n}{b_n}\}\), satisfies
		\[
		\hat{a}_n \leq \hat{a}_{n+1}(1 + w \hat{a}_{n+1})
		\]
		for some sufficiently small constant \(w>0\), and  the sequences \(\{\frac{b_n^3}{a^{(\Gamma)}_n} |\beta^{(\Gamma)}_n|^2\}\) and \(\{\frac{b_n^3 |G_n|^2}{(a^{(\Gamma)}_n)^3}\}\) are non-decreasing in $n$, then there exists an increasing sequence \(\{\hat{\eta}^{(\Gamma)}_n\}\) and a constant \(c'>0\) such that
		\[
		\E[|\Gamma_{n+1} - \Gamma_{n+1}^*|^2] \leq c' \hat{a}_n \hat{\eta}^{(\Gamma)}_n.
		\]
		In particular, if the parameters \(b_n, a^{(\Gamma)}_n, c^{(\bm\phi)}_n, c^{(\Gamma)}_n, \beta^{(\Gamma)}_n\) are set as in Remark \ref{remark_example_gamma}, then
		\[
		\E[|\Gamma_{n}|^2] \leq c \frac{(e \vee \log n)^{p_2} (1 \vee \log\log n)^{\frac{4}{3}}}{n^{\frac{1}{2}}}
		\]
		where \(p_2\) is the same constant appearing in Theorem \ref{thm_Gamma_convergence}.
	\end{theorem}

    The proof of this theorem is lengthy and here we highlight the key ideas while providing the full proof in Appendix~\ref{appendix:proof_thm3}. Consider the squared error $\rho_n := \mathbb{E}[|\Gamma_n-\Gamma^*_n|^2]$ and use the recursion for $\Gamma_n$ 
    to derive a recursive  inequality of the form
    $\rho_{n+1} \le (1 - c \hat{a}_n)\rho_n + 4 \hat{a}_n^2 \hat{\eta}^{(\Gamma)}_n$
    for a suitably defined step-size sequence $\{\hat{a}_n\}$ and auxiliary sequence $\{\hat{\eta}^{(\Gamma)}_n\}$. Applying  the monotonicity of these sequences and adapting the stochastic-approximation rate analysis in \cite{broadie2011general} yield the asserted bound.

	\subsection{Convergence Analysis of \texorpdfstring{$\bm\phi_n$}{φₙ}}
	\label{sec_convergence_phi}
	
	In this subsection, we investigate the convergence properties of the actor parameter \(\bm\phi_n\) for both cases $x_0 \neq 0$ and $x_0=0$.
	
	\begin{theorem}
		\label{thm_convergence_and_rate_phi}
		Under the same setting of Theorem \ref{thm_convergence_and_rate_gamma}, we have
		
		\begin{enumerate}[label=(\alph*)]
			\item As \(n \to \infty\), \(\bm\phi_n\) converges almost surely to
			\(
			\bm\phi^* =  -\Bigg(\sum_{j=1}^{m} D_j D_j^\top\Bigg)^{-1}\Bigg(B + \sum_{j=1}^{m}C_jD_j\Bigg).
			\)
			\item The expected squared error satisfies
			\[
			\mathbb{E} [|\bm\phi_n - \bm\phi^*|^2] \leq c \frac{(\log n)^{p_1} (\log \log n)^{\frac{4}{3}}}{n^{\frac{1}{2}}}, \quad \text{if } x_0 \neq 0,
			\]
			\[
			\mathbb{E} [|\bm\phi_n - \bm\phi^*|^2] \leq c \frac{(\log n)^{p_1'} (\log \log n)^{\frac{4}{3}}}{n^{\frac{1}{4}}}, \quad \text{if } x_0 = 0,
			\]
    where \(c\), \(p_1\), and \(p_1'\) are positive constants only dependent on the model parameters.
		\end{enumerate}
	\end{theorem}

    A proof proceeds by separating the cases of $x_0 \neq 0$ and $x_0 = 0$. When $x_0 \neq 0$, one adapts the projected stochastic-approximation framework used in \cite[Theorem 4.1]{huang2024sublinear} to our setting. When $x_0 = 0$, one applies the core ideas from the proof of Theorem~1. The full proof is provided in Appendix~\ref{appendix:proof_thm4}.

	The established convergence of the learned policy underpins the regret analysis of Algorithm \ref{algo_rl_lq}. Theorem \ref{thm_convergence_and_rate_phi} posits that the MSE of \(\bm\phi_n\) depends on whether the initial state \(x_0\) is zero or nonzero; yet both cases ultimately lead to the same sublinear regret bound, which will be shown in the next subsection.

	\subsection{Regret Bound}
	\label{sec_regret}
	
	The regret quantifies the difference between the value function of the learned policy and that of the oracle optimal policy. We say that a regret bound is sublinear if the regret grows strictly slower than linearly in the number of iterations, $N$. The average  regret over $N$ iterations of an algorithm with sublinear regret converges to 0 as $N\rightarrow \infty$, yielding an almost  oracle performance in the long run.

	For a stochastic Gaussian policy \(\pi = \mathcal{N}(\cdot|\bm\phi x, \Gamma)\), we denote
	\begin{equation}
		\begin{aligned}
			\label{jbar}
			\Bar{J}(\bm\phi, \Gamma) = \mathbb{E} \biggl[& \int_0^T \left(-\frac{1}{2}Qx^\pi(s)^2  \right) \mathrm{d}s \\
			&- \frac{1}{2}Hx^\pi(T)^2 \Big| x^\pi(0)=x_0 \biggr].
		\end{aligned}
	\end{equation}
	
	This function evaluates policy \(\mathcal{N}(\cdot|\bm\phi x, \Gamma)\) under the original objective \textit{without} entropy regularization. The oracle value of the original problem is \(\Bar{J}(\bm\phi^*, \bm{0})\).

	\begin{theorem}
		\label{thm_regret}
		Under the same setting of Theorem \ref{thm_convergence_and_rate_gamma}, Algorithm \ref{algo_rl_lq} leads to the following regret bounds over \(N\) iterations:
		\[
		\begin{aligned}
			&\sum_{n=1}^{N}\mathbb{E} \big[\Bar{J}(\bm\phi^*, \bm{0}) - \Bar{J}(\bm\phi_n, \Gamma_n)\big] \\
			\leq& c + c N^{\frac{3}{4}}(\log N)^{\frac{p_1+p_2}{2}+1} (\log\log N), \quad \text{if } x_0\neq0,\\
			&\sum_{n=1}^{N}\mathbb{E} \big[\Bar{J}(\bm\phi^*, \bm{0}) - \Bar{J}(\bm\phi_n, \Gamma_n)\big] \\
			\leq& c+c N^{\frac{3}{4}} (\log N)^{\frac{p_1'+p_2}{2} \vee (p_1'+\frac{3}{2})} (\log\log N), \quad \text{if } x_0= 0,
		\end{aligned}
		\]
		where \(c > 0\) is a constant independent of \(N\), and \(p_1, p_1', p_2\) are the same constants given in Theorems \ref{thm_convergence_and_rate_gamma} and \ref{thm_convergence_and_rate_phi}.
	\end{theorem}

    We defer the full proof to Appendix~\ref{appendix:proof_thm5}, while outlining the key ideas here. Using a Feynman--Kac representation, we express the value function under the Gaussian policy as a smooth function of $(\phi,\Gamma)$ and decompose the one-step regret into a small number of terms associated with the critic and actor errors, treating the cases $x_0 \neq 0$ and $x_0 = 0$ separately. Each term is then bounded using H\"older- and Cauchy--Schwarz-type inequalities together with the mean-squared convergence rates from Theorems~1 and~4. Summing these bounds yields the stated $O(N^{3/4})$ regret rate.

	
	


	\begin{remark}
		\label{remark_discretization_error}
		Discretization errors arising from the time step size \(\Delta t_n\)  are captured by the terms \(\beta^{(\bm\phi)}_n\) and \(\beta^{(\Gamma)}_n\). Theorems \ref{thm_Gamma_convergence} - \ref{thm_convergence_and_rate_phi} suggest that \(\beta^{(\bm\phi)}_n\) and \(\beta^{(\Gamma)}_n\) should be set to decrease at the order of \(n^{-\frac{3}{8}}\), which can be achieved by setting \(\Delta t_n=T(n+1)^{-\frac{5}{8}}\). We omit the details here, but refer to \cite{kloedennumerical,szpruch2024optimal,huang2024sublinear} for discussions on  time-discretization analyses.
	\end{remark}

	\section{Numerical Experiments}
	\label{sec_experiments}
	
	This section reports four numerical experiments to evaluate the performance of our proposed exploratory adaptive LQ-RL algorithm. The first one validates our theoretical findings by examining the convergence behaviors of \(\bm\phi\) and \(\Gamma\) and confirming the sublinear regret bound. The second one compares our model-free approach with a model-based benchmark \cite{szpruch2024optimal} adapted to our setting that incorporates state- and control-dependent volatility. The third one evaluates the effect of exploration strategies by comparing our approach to a deterministic exploration schedule \cite{huang2024sublinear}. We examine scenarios where exploration parameters are either overly conservative or excessively aggressive, demonstrating the advantages of dynamically adjusting exploration based on observed data rather than relying on predefined schedules. The final experiment extends the third one by introducing random model parameters and random initial exploration parameter values, capturing real-world situations where no prior knowledge of the environment is available. This last experiment further demonstrates the effectiveness of the data-driven exploration mechanism in adapting to the environment.

	For all the experiments, we set the control dimension and Brownian motion dimension to be \(l = 1\) and \(m = 1\). Under this setting, the LQ dynamics are simplified to
	\begin{equation}
		\label{eq_classical_dynamics_one_dim}
		\mathrm{d}x^u(t) = (A x^u(t) + B u(t)) \mathrm{d}t + (C x^u(t) + D u(t)) \mathrm{d}W(t).
	\end{equation}
	Moreover, we set all the model parameters \(A, B, C, D, Q, H, x_0,\) and \(T\) to be 1 and use a time step of \(\Delta t = 0.01\) for the first three experiments.
	See also Appendix \ref{sec_experiment_details} for the replicability of our experiments.

	\subsection{Numerical Validation of Theoretical Results}
	
	To validate the theoretical results on the convergence rates and regret bounds, we focus on the actor parameters \(\Gamma\) and \(\bm\phi\). We conduct 100 independent runs, each consisting of 100,000 iterations to capture long-term behavior. Convergence rates and regret are plotted in Figure \ref{fig:mf} using a log-log scale, where the logarithm of the MSE (for convergence) or regret is plotted against the logarithm of the number of iterations.
	
	\begin{figure}[htbp]
		\centering
		\begin{subfigure}[t]{0.31\linewidth}
			\centering
			\captionsetup{justification=centering}
			\includegraphics[width=\linewidth]{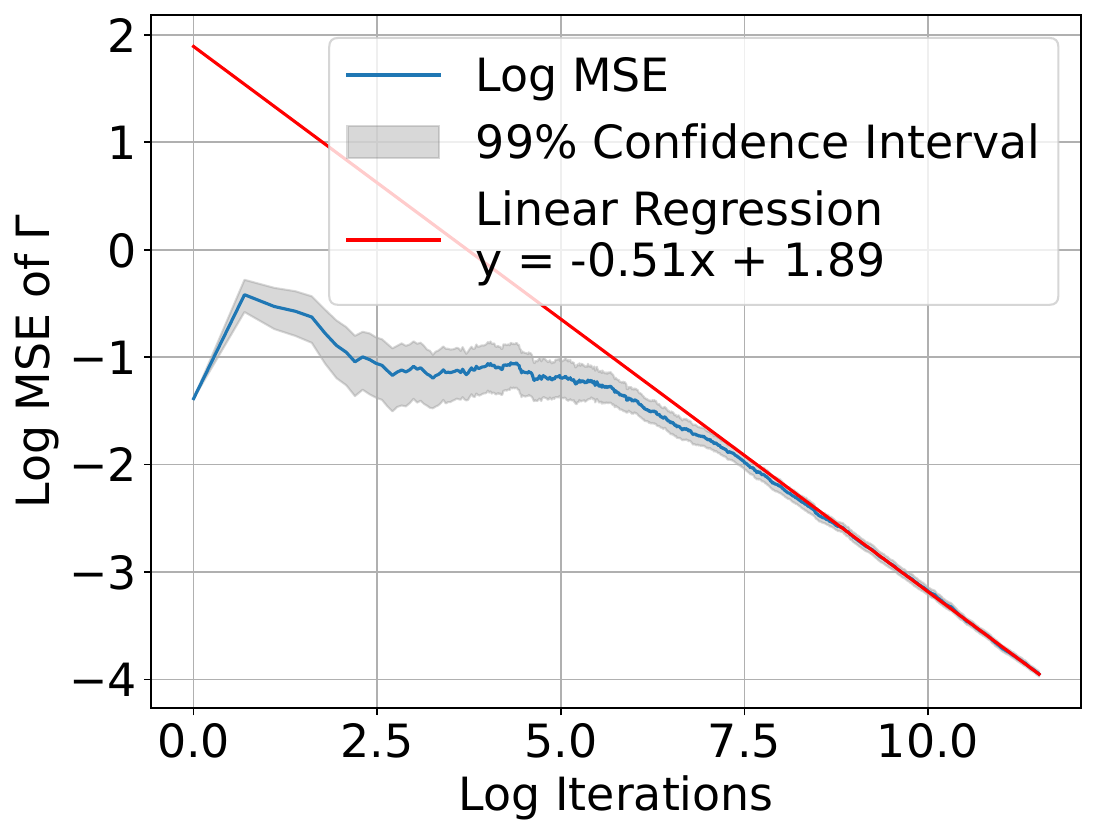}
			\caption{MSE of $\Gamma$}
			\label{figure_mf_1}
		\end{subfigure}\hspace{0.025\linewidth}%
		\begin{subfigure}[t]{0.31\linewidth}
			\centering
			\captionsetup{justification=centering}
			\includegraphics[width=\linewidth]{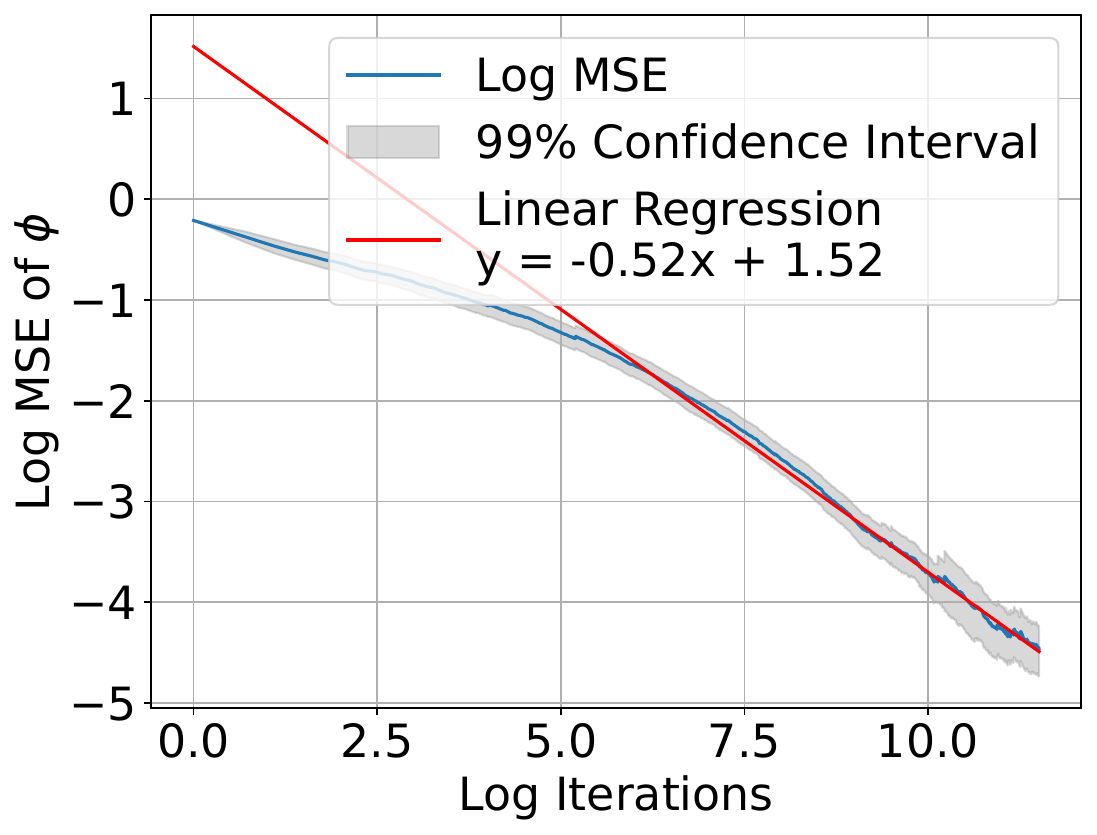}
			\caption{MSE of $\bm\phi$}
			\label{figure_mf_2}
		\end{subfigure}\hspace{0.025\linewidth}%
		\begin{subfigure}[t]{0.31\linewidth}
			\centering
			\captionsetup{justification=centering}
			\includegraphics[width=\linewidth]{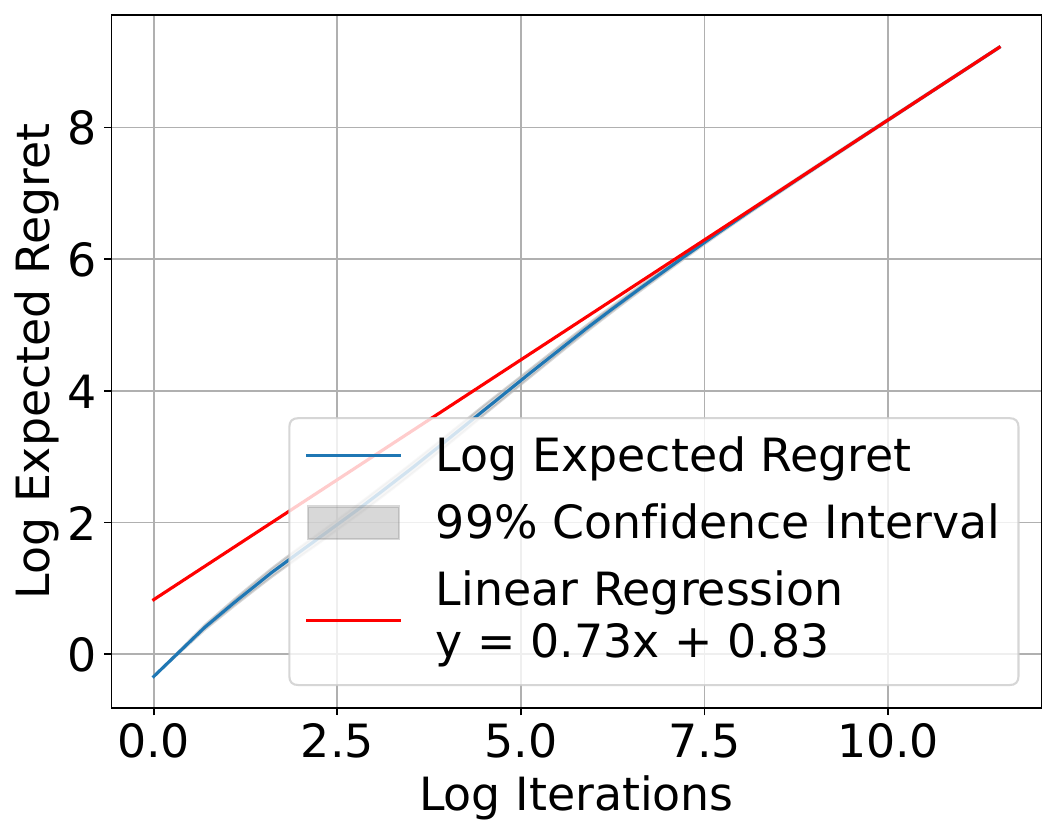}
			\caption{Regret}
			\label{figure_mf_3}
		\end{subfigure}
		\caption{\textbf{ Log-log plot of Algorithm \ref{algo_rl_lq}.} }
		\label{fig:mf}
	\end{figure}
	
	The theoretical convergence rates  in Theorems~\ref{thm_convergence_and_rate_gamma} and~\ref{thm_convergence_and_rate_phi} have both orders of \(n^{-1/2}\), while the regret bound in Theorem \ref{thm_regret} is of order \(N^{3/4}\), all up to logarithmic factors. On a log--log plot,
    a rate of order \(n^{\alpha}\) appears as an approximately straight line with slope \(\alpha\). With linear regression, which is commonly used to estimate empirical convergence rates or regrets (see, e.g., \cite{shamir2013stochastic,nie2022explore}), Figures~\ref{figure_mf_1} and~\ref{figure_mf_2} show slopes of about \(-0.51\) and \(-0.52\) for \(\Gamma_n\) and \(\bm\phi_n\), respectively, and Figure~\ref{figure_mf_3} indicates a regret slope of about \(0.73\), all of which are close to the theoretical values \(-\tfrac{1}{2}\) and \(\tfrac{3}{4}\). Therefore, these numerical results are consistent with the theoretical analysis and confirm the long-term effectiveness of our learning algorithm.

	\subsection{Model-Free vs. Model-Based}
	
	Under the same setting, we compare our model-free LQ-RL algorithm with data-driven exploration to the recently developed model-based method with a deterministic exploration schedule \cite{szpruch2024optimal}, which estimates parameters \(A\) and \(B\) in the drift term under the assumption of a constant volatility. To cover the settings with state- and control-dependent volatilities, \cite{huang2024sublinear} modified  the algorithm in \cite{szpruch2024optimal} by incorporating estimations of also \(C\) and \(D\), which we adopt in our experiment. 
	
	\begin{figure}[htbp]
		\centering
		\begin{subfigure}[t]{0.40\linewidth}
			\centering
			\captionsetup{justification=centering}
			\includegraphics[width=\linewidth]{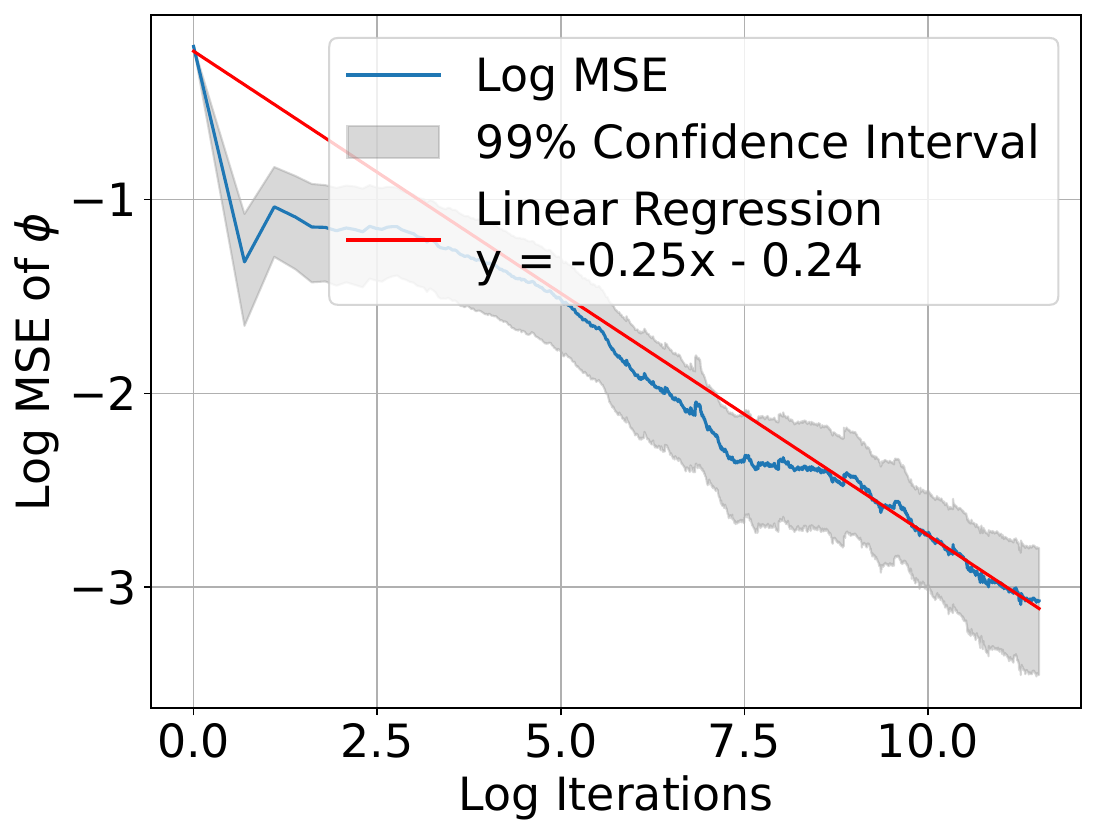}
			\caption{MSE of $\bm\phi$ in model-based algorithm}
			\label{figure_mb_1}
		\end{subfigure}\hspace{0.045\linewidth}%
		\begin{subfigure}[t]{0.40\linewidth}
			\centering
			\captionsetup{justification=centering}
			\includegraphics[width=\linewidth]{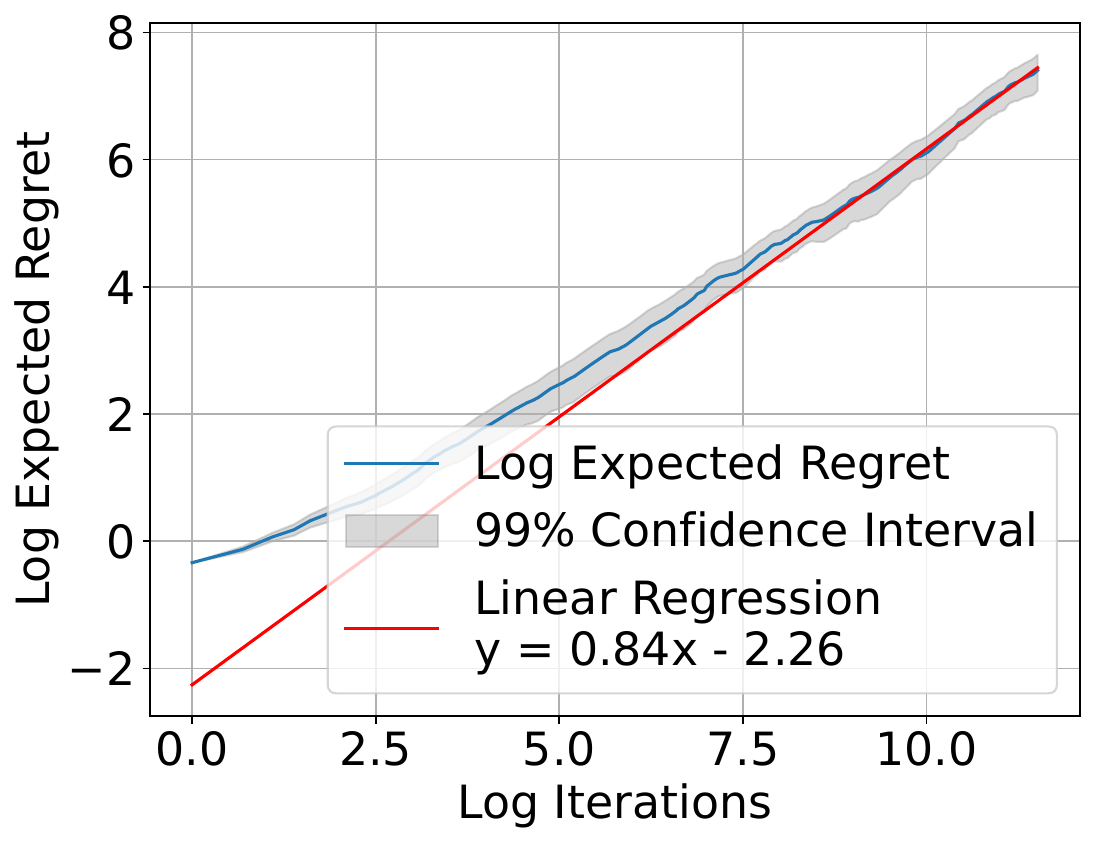}
			\caption{Regret of model-based algorithm}
			\label{figure_mb_3}
		\end{subfigure}
		\caption{\textbf{Log-log plot of model-based LQ-RL algorithm with fixed exploration schedule.} }
		\label{fig:mb}
	\end{figure}
	
	Figure \ref{figure_mb_1} shows that the model-based counterpart exhibits a significantly slower convergence rate for \(\bm\phi\) with a slope of \(-0.25\). Moreover, Figure \ref{figure_mb_3} indicates a worse regret with a slope of \(0.84\), considerably larger than that  of ours.\footnote{\(\Gamma\) follows a fixed schedule in the model-based benchmark; hence its plot is uninformative and  omitted.}

    These differences are also summarized in Table~\ref{table:model_free_vs_model_based} for a direct comparison.

    \begin{table}[t]
    \centering
    \begin{tabular}{lcc}
    \toprule
    Method & MSE slope of \(\bm\phi_n\) & Regret slope \\
    \midrule
    Model-free (this paper) & \(-0.52\) & \(0.73\) \\
    Model-based benchmark   & \(-0.25\) & \(0.84\) \\
    \bottomrule
    \end{tabular}
    \caption{Empirical slopes (log--log) of the MSE of \(\bm\phi_n\) and the regret for the model-free and model-based algorithms.}
    \label{table:model_free_vs_model_based}
    \end{table}

	\subsection{Adaptive vs. Fixed Explorations}
	\label{sec_exploration_experiment}
	
	Now, we compare two model-free continuous-time LQ-RL algorithms: ours with data-driven exploration (LQRL\_Adaptive) and the one by \cite{huang2024sublinear} with fixed exploration schedules (LQRL\_Fixed). Both algorithms follow the same fundamental framework, with the key distinction being how exploration is carried out. In Section \ref{sec_limitations}, we outlined several drawbacks associated with deterministic exploration schedules. To illustrate them more concretely, we analyze two scenarios where a pre-determined exploration is either excessive or insufficient.
	
	For comparison, we examine the trajectories of the actor exploration parameter \(\Gamma\) and the regrets over iterations.\footnote{The temperature parameter \(\gamma\) is not compared, as it remains fixed in \cite{huang2024sublinear}.} Both algorithms share exactly the same experimental setup, and we conduct 1,000 independent runs to observe the overall performances.
	
	\paragraph{Excessive Exploration with a Near-Optimal Policy}
	We first consider a scenario where the policy parameter is initialized at \(\bm\phi_0 = -1.8\), which is close to its optimal value \(\bm\phi^* = -2\). As the initial policy is already nearly optimal, exploration is not highly necessary. However, the initial exploration levels by both the critic and actor are set excessively high at \(\gamma_0 = \Gamma_0 = 20\).
	
	\begin{figure}[htbp]
		\centering
		\begin{subfigure}[t]{0.40\linewidth}
			\centering
			\captionsetup{justification=centering}
			\includegraphics[width=\linewidth]{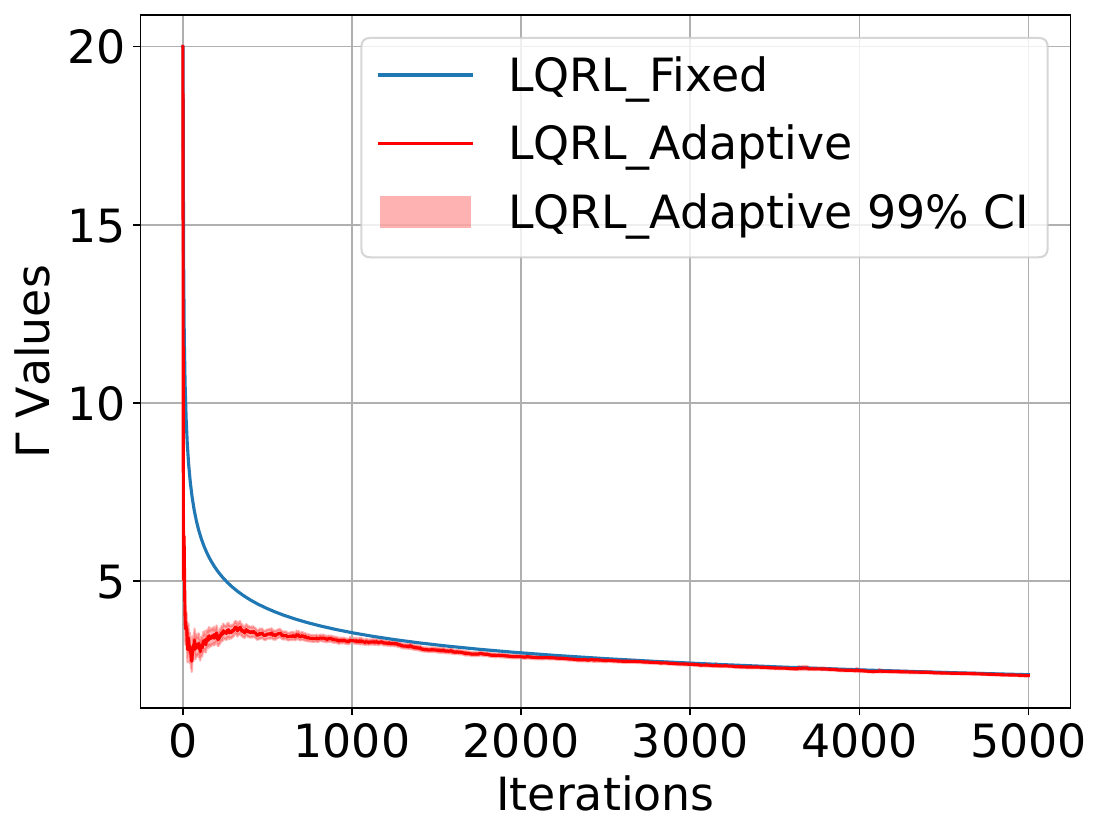}
			\caption{Trajectory of \(\Gamma\)}
			\label{figure_less_1}
		\end{subfigure}\hspace{0.045\linewidth}%
		\begin{subfigure}[t]{0.40\linewidth}
			\centering
			\captionsetup{justification=centering}
			\includegraphics[width=\linewidth]{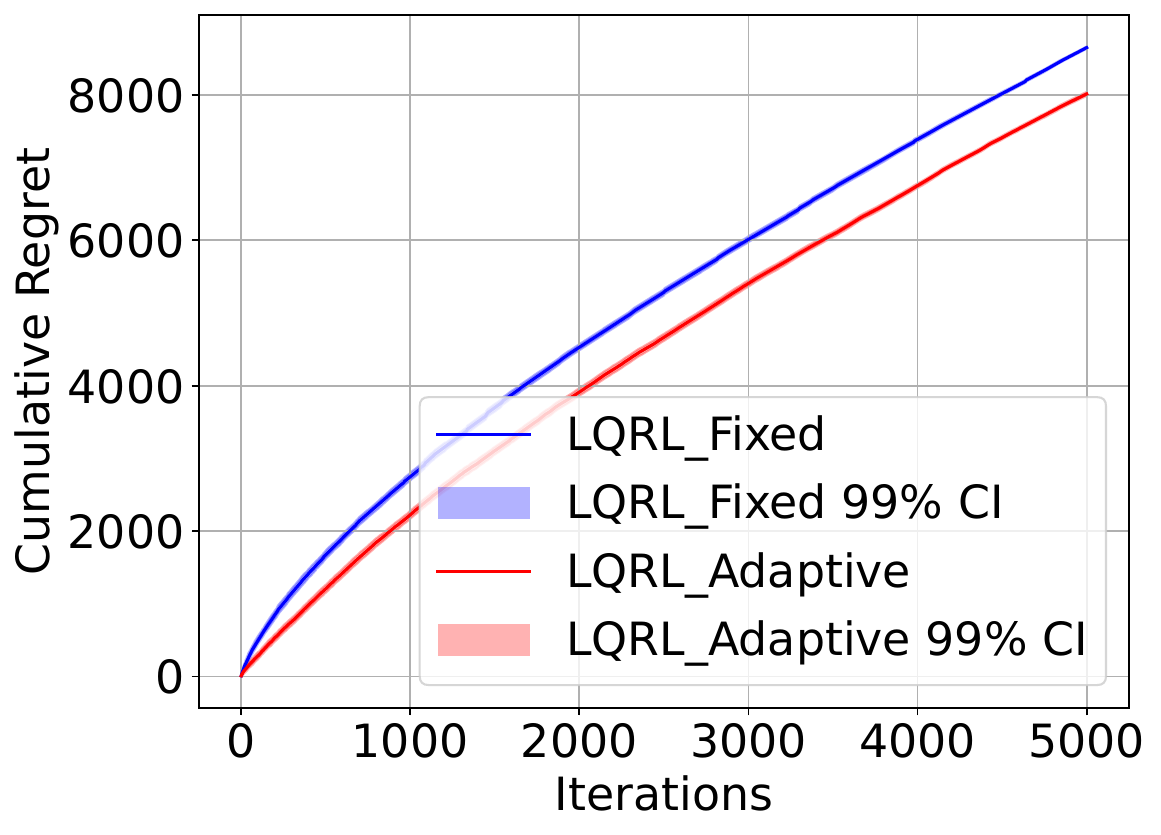}
			\caption{Regret}
			\label{figure_less_2}
		\end{subfigure}
		\caption{Comparison of exploration level and regret under excessive initial exploration.}
		\label{fig:less}
	\end{figure}
	
	Figure \ref{figure_less_1} shows that while both algorithms start with an unnecessarily  high level of \(\Gamma\), LQRL\_Adaptive rapidly adjusts \(\Gamma\) downward, having effectively learned that exploration is less needed in this setting. By contrast, LQRL\_Fixed follows a predetermined decay in \(\Gamma\) and remains at a consistently higher level. As a result, Figure \ref{figure_less_2} demonstrates that LQRL\_Adaptive achieves lower regret, indicating a more efficient learning process by dynamically adapting exploration to the circumstances.

	\paragraph{Insufficient Exploration with a Poor Policy}
	We now examine an opposite scenario where \(\bm\phi_0\) starts at \(0\), relatively far away from its optimal value \(\bm\phi^*=-2\). The initial exploration parameters are set too low at \(\gamma_0 = \Gamma_0 = 0.02\), creating a situation where increased exploration is crucial for effective learning.
	
	\begin{figure}[htbp]
		\centering
		\begin{subfigure}[t]{0.40\linewidth}
			\centering
			\captionsetup{justification=centering}
			\includegraphics[width=\linewidth]{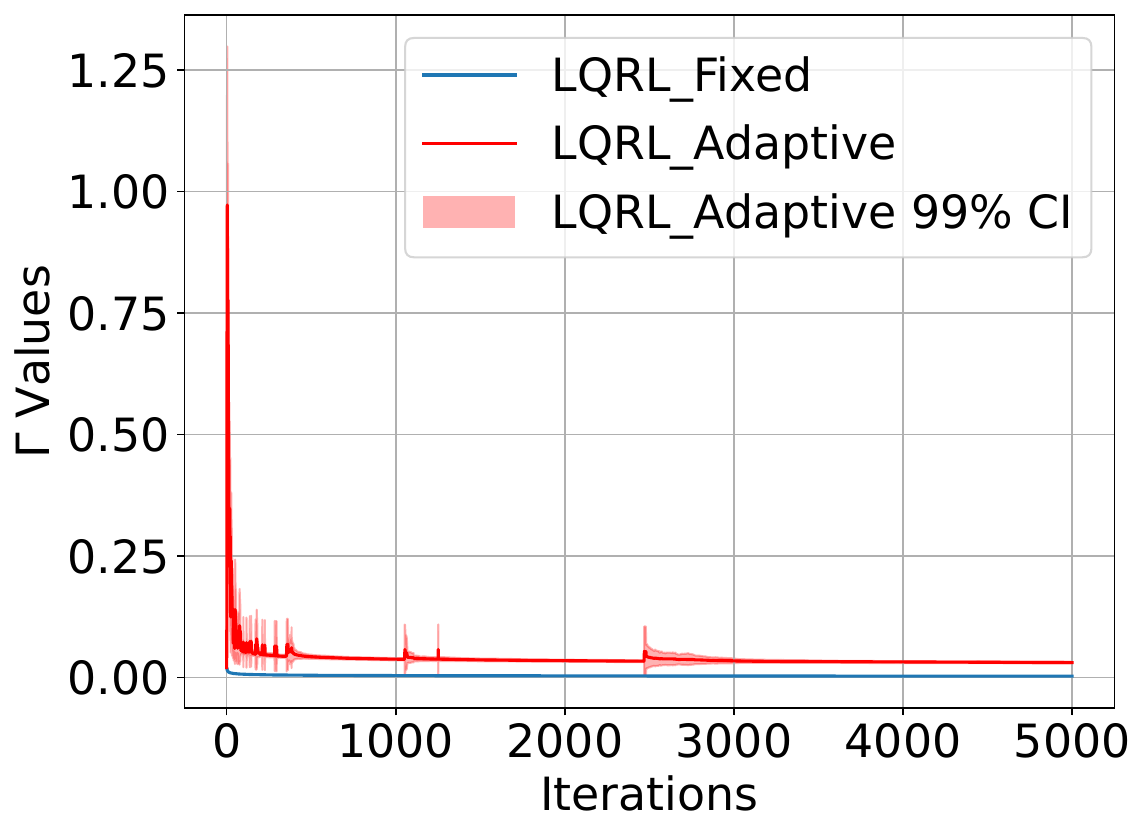}
			\caption{Trajectory of \(\Gamma\)}
			\label{figure_more_1}
		\end{subfigure}\hspace{0.045\linewidth}%
		\begin{subfigure}[t]{0.40\linewidth}
			\centering
			\captionsetup{justification=centering}
			\includegraphics[width=\linewidth]{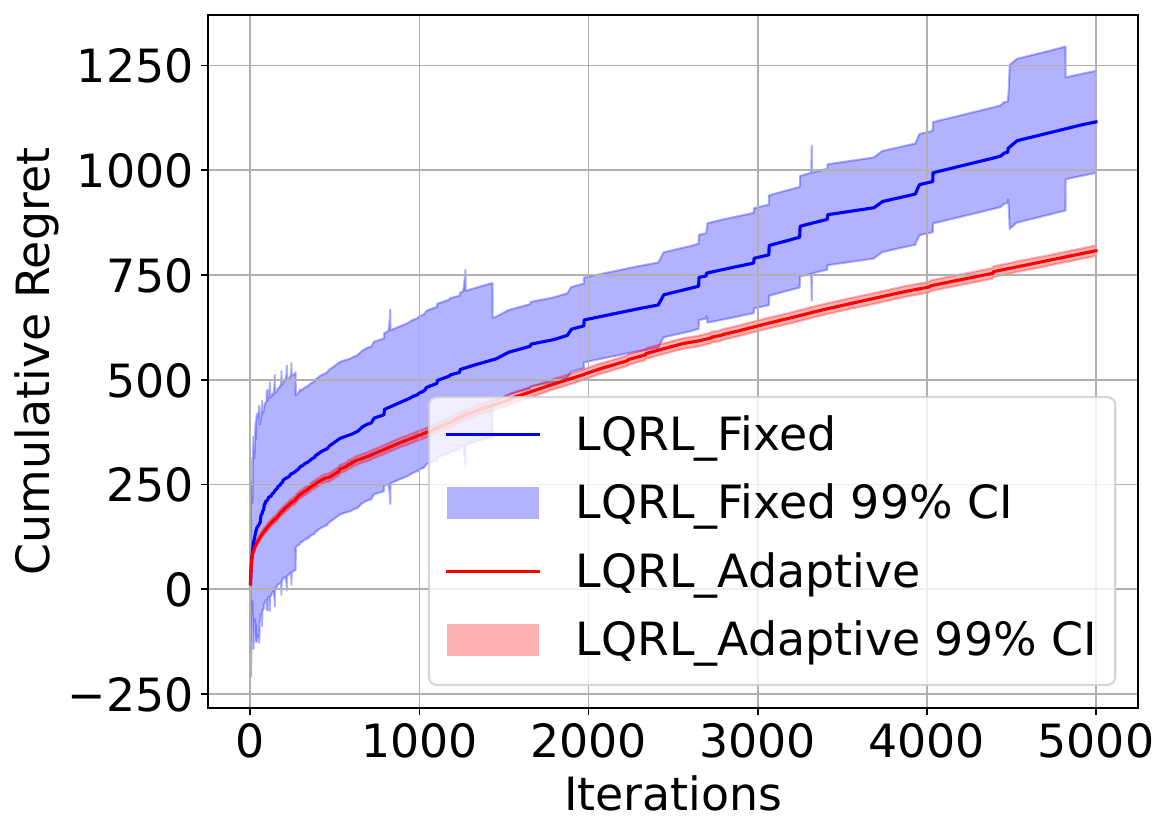}
			\caption{Regret}
			\label{figure_more_2}
		\end{subfigure}
		\caption{Comparison of exploration level and regret under insufficient initial exploration.}
		\label{fig:more}
	\end{figure}
	
	Figure \ref{figure_more_1} shows that LQRL\_Adaptive rapidly increases \(\Gamma\) from its initial low value of \(0.02\) to nearly \(1\), maintaining a higher exploration level than LQRL\_Fixed throughout. By contrast, LQRL\_Fixed follows a predetermined {\it decaying} schedule, reducing \(\Gamma\) even further over iterations, which is \textit{counterproductive} in this case where greater exploration is called for. As a result, Figure \ref{figure_more_2} demonstrates an even larger performance gap in regret compared to the previous scenario, further underscoring the benefits of adaptively adjusting exploration for learning.

	\subsection{Adaptive vs. Fixed Explorations: Randomized Model Parameters}
	\label{sec_random_experiment}
	
	In the final experiment, we evaluate the robustness of the comparison between LQRL\_Adaptive and LQRL\_Fixed in a setting where model parameters and initial exploration levels are randomly generated, mimicking real-world scenarios with no prior knowledge of the environment and no pre-tuning of exploration parameters. Specifically, for each simulation run, the model parameters \(A, B, C,\) and \(D\) are sampled from a uniform distribution \(\mathcal{U}(-5, 5)\), while the initial exploration levels of \(\gamma_0\) and \(\Gamma_0\) follow \(\mathcal{U}(0, 5)\). The initial policy parameter is set to be \(\bm\phi_0 = 0\) as a neutral starting point. To ensure the reliability of the experiment, we conduct \(10,000\) independent runs.
	
	\begin{figure}[htbp]
		\centering
		\includegraphics[width=0.45\linewidth]{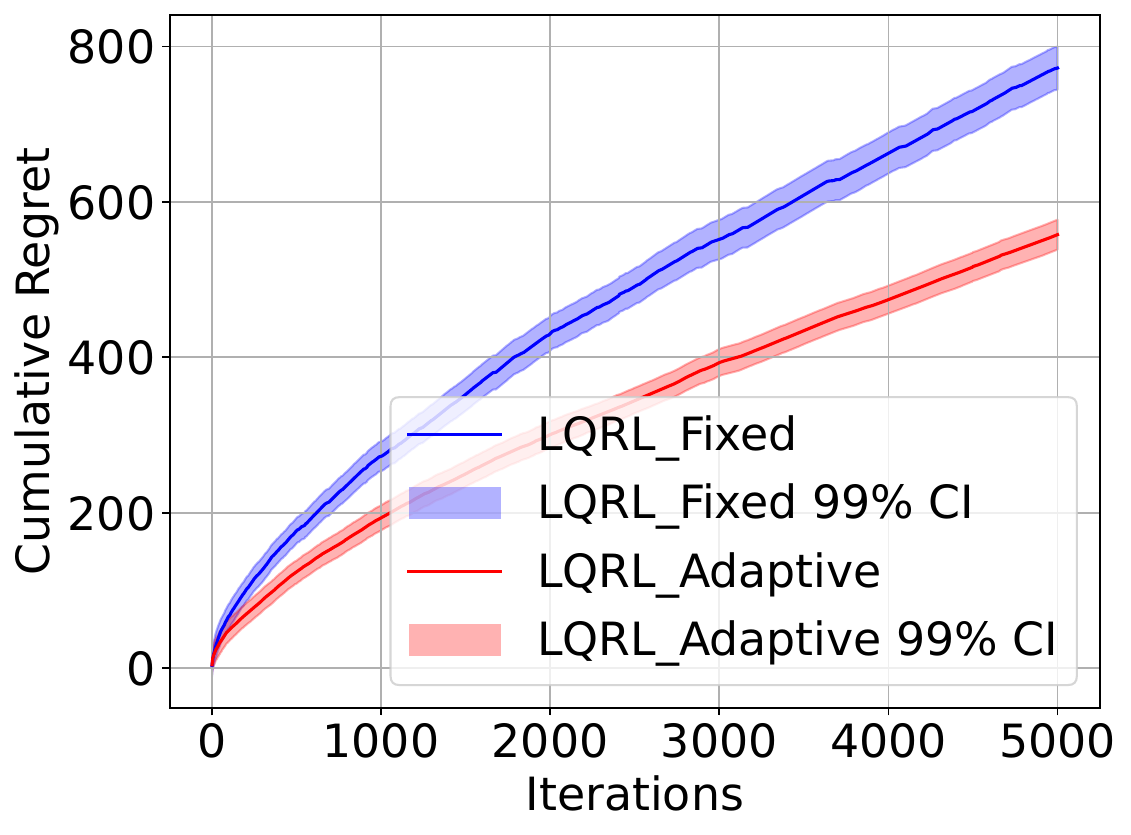}
		\caption{Comparison of regrets under randomized environments.}
		\label{fig:random}
	\end{figure}
	
	Figure~\ref{fig:random} demonstrates that LQRL\_Adaptive significantly outperforms LQRL\_Fixed, with the performance gap widening as the number of iterations increases. While the former is expected to be more computationally costly due to multiple adaptive update steps, its wall-clock runtime was less than $5\%$ higher than that of the latter in this experiment  (approximately $4$h$25$ vs.\ $4$h$13$ on our machine). So the extra updates of $\gamma_n$ and $\Gamma_n$ incur negligible computational overhead vis-\`a-vis the substantial performance gain.

	\section{Conclusions}
	\label{sec_conclusion}
	
This paper is a continuation of \cite{huang2024sublinear}, aiming to develop a model-free continuous-time RL framework for a class of indefinite stochastic LQ problems with state- and control-dependent volatilities. A key contribution, compared to \cite{huang2024sublinear}, is the introduction of a data-driven exploration mechanism that adaptively adjusts both the temperature parameter controlled by the critic and the stochastic policy variance managed by the actor. This approach overcomes the limitations of fixed exploration schedules, which often require extensive tuning and incur unnecessary exploration costs. While the theoretically proved sublinear regret bound of $O(N^{\frac{3}{4}})$ is the same as that in \cite{huang2024sublinear}, numerical experiments confirm that our adaptive exploration strategy significantly improves learning efficiency compared to \cite{huang2024sublinear}.

Research on {\it model-free} continuous-time RL is still in the very early innings. To our best knowledge  \cite{huang2024sublinear,tang2024regret} are the only works on regret analysis in the realm, largely because the problem is extremely challenging.
We focus on the same LQ problem in \cite{huang2024sublinear} because this is the problem that we are able to solve at the moment, one nevertheless that may serve as a starting point for tackling more complex problems. Meanwhile, it remains an interesting open question whether adaptive exploration strategies also theoretically improve the regret bound, which is observed from our numerical experiments.
Besides serving as a theoretical benchmark, the class of LQ problems studied here arises naturally in several applied domains. A prominent example is continuous-time mean--variance portfolio selection \cite{huang2024mean}, where the state variable, namely the wealth process, is inherently one-dimensional and the problem is automatically of an LQ structure. Other examples include surplus or cash-flow management in finance and production or inventory control in manufacturing. In those cases, the proposed model-free actor--critic method can be applied directly using sampled and observed trajectories, without requiring model identification.
Finally, our framework relies on entropy regularization and policy variance for exploration, which may be enhanced by goal-based constraints/rewards to improve sampling efficiency. All these provide promising avenues for future research, driving further advancements in data-driven exploration and continuous-time RL.

	\begin{appendices}

        \section{Proof of Theorem 2}
        \label{appendix:proof_thm2}
        \begin{proof}
		It is straightforward to see that \(\gamma_n\) almost surely converges to \(0\).
		To prove the almost sure convergence of \(\Gamma_n\), the key idea  is to establish inductive upper bounds on \(|\Gamma_n|^2\), namely to  bound \(|\Gamma_{n+1}|^2\) by  a function of \(|\Gamma_n|^2\).
		
		
		Define the deviation term \(U^{(\Gamma)}_n = \Gamma_n - \Gamma_n^*\), where \(\Gamma_n^*\) is given by:
		\begin{equation}
			\label{eq_phi_2n_star0}
			\Gamma_n^* = \frac{\gamma_n T}{\int_0^T k_1(t; \bm\theta_n) \dd t} \Mj.
		\end{equation}
		
		By the adaptive temperature parameter \(\gamma_n\) in \eqref{eq_gamma_definition}, \(\Gamma_n^*\) can be rewritten as:
		\begin{equation}
			\label{eq_phi_2n_star}
			\Gamma_n^* = \frac{c_\gamma}{b_n} \Mj,
		\end{equation}
		where we recall \(c_\gamma > 0\) is such that  \(\frac{1}{c_\gamma}\) is smaller than the minimum eigenvalue of \(\Mj\), denoted as \(\lambda_{\min}(\Mj)\). Thus \(\Gamma_n^* > \frac{1}{b_n} I\). Since \(b_n \uparrow \infty\), \(\Gamma_n^*\) decreases in $n$, leading to the positive matrix difference:
		\(
		G_n = \Gamma_n^* - \Gamma_{n+1}^* > 0.
		\)
		
		We now show the almost sure convergence of \(U^{(\Gamma)}_n\) to 0.  Consider sufficiently large \(n\) such that \(\Gamma_n^* \in K^{(\Gamma)}_{n+1}\). By applying the general projection inequality $|\Pi_K(y) - x|^2 \leq |y - x|^2$, we have
		
		\[\begin{aligned}
			& \E\left[|U^{(\Gamma)}_{n+1}|^2 \Big| \g_n \right] \\
			\leq & \E\left[| U^{(\Gamma)}_n -  a^{(\Gamma)}_n[ h^{(\Gamma)}(\Gamma_n; \bm\theta_n, \gamma_n) + \xi^{(\Gamma)}_n] + G_n |^2 \Big| \g_n \right] \\
			\leq & |U^{(\Gamma)}_n|^2 - 2a^{(\Gamma)}_n \langle U^{(\Gamma)}_n,  h^{(\Gamma)}(\Gamma_n; \bm\theta_n, \gamma_n) \rangle + (1 + |U^{(\Gamma)}_n|^2) \\
			&(a^{(\Gamma)}_n | \beta^{(\Gamma)}_n| + |G_n|) + 4(a^{(\Gamma)}_n)^2 \biggl( |h^{(\Gamma)}(\Gamma_n; \bm\theta_n, \gamma_n)|^2 \\
			&+ | \beta^{(\Gamma)}_n|^2 + |\frac{G_n}{a^{(\Gamma)}_n}|^2 + \E\left[ \left|\xi^{(\Gamma)}_n -  \beta^{(\Gamma)}_n \right|^2 \Big| \g_n\right] \biggr).
		\end{aligned}\]
		
		Recall that $\frac{1}{b_n}| \leq |\Gamma_n| \leq c^{(\Gamma)}_n$ almost surely. By \eqref{eq_projection_sets_critic}, \eqref{eq_projection_sets_actor}, \eqref{eq_h2n_upper_app}, and \eqref{eq_Gamma_noise_assumption}, we can further derive that
		\[\begin{aligned}
			& \E\left[|U^{(\Gamma)}_{n+1}|^2 \Big| \g_n \right] \leq (1 + \kappa^{(\Gamma)}_n) |U^{(\Gamma)}_n|^2  - \zeta^{(\Gamma)}_n + \eta^{(\Gamma)}_n,
		\end{aligned}\]
		where
		$\kappa^{(\Gamma)}_n = a^{(\Gamma)}_n | \beta^{(\Gamma)}_n| + |G_n|$, \\
		$\zeta^{(\Gamma)}_n = 2a^{(\Gamma)}_n \langle U^{(\Gamma)}_n,  h^{(\Gamma)}(\Gamma_n; \bm\theta_n, \gamma_n) \rangle$, and
		\[
		\begin{aligned}
			\eta^{(\Gamma)}_n =& a^{(\Gamma)}_n | \beta^{(\Gamma)}_n| + |G_n| + 4(a^{(\Gamma)}_n)^2\biggl\{ c (1 + (c^{(\Gamma)}_n)^4)\\
			&+ c (1 + (c^{(\bm\phi)}_n)^8 + (c^{(\Gamma)}_n)^8 + (\log b_n)^8)\exp{\{c(c^{(\bm\phi)}_n)^4\}} \\
			&+ | \beta^{(\Gamma)}_n|^2 + |\frac{G_n}{a^{(\Gamma)}_n}|^2 \biggl\}.
		\end{aligned}
		\]
		
		Next, we prove that \(\sum |G_n| < \infty\) and \(\sum |G_n|^2 < \infty\). By the assumption (iii) in \eqref{eq_Gamma_assumption_app} along with \eqref{eq_phi_2n_star}, we obtain
		\begin{equation}
			\label{eq_g_n}
			\begin{aligned}
				\sum_{n=0}^{\infty} |G_n| &= \sum_{n=1}^{\infty} |\Gamma_n^* - \Gamma_{n+1}^*| < c \Sj |\frac{1}{b_n} - \frac{1}{b_{n+1}}| < \infty.\\
			\end{aligned}
		\end{equation}
		
		Furthermore, since $b_n \uparrow \infty$, $n_0 = \inf{\{n^\prime \in \mathbb{N}: |G_n|<1 \text{ for all } n\geq n^\prime\}}<\infty$.
		Therefore, \eqref{eq_g_n} yields
		\begin{equation}
			\label{eq_g_n2}
			\begin{aligned}
				\sum_{n=1}^{\infty} |G_n|^2 < \sum_{n=1}^{n_0-1} |G_n|^2 + \sum_{n=n_0}^{\infty} |G_n| < \infty.
			\end{aligned}
		\end{equation}
		
		By the assumptions (i)-(ii) of \eqref{eq_Gamma_assumption_app}, as well as \eqref{eq_g_n} and \eqref{eq_g_n2}, we know $\sum \kappa^{(\Gamma)}_n<\infty$ and $\sum \eta^{(\Gamma)}_n<\infty$. It then follows from \cite[Theorem 1]{robbins1971convergence} that $\left|U^{(\Gamma)}_n\right|^2$ converges to a finite limit and $\sum \zeta^{(\Gamma)}_n<\infty$ almost surely.
		
		It suffices to show $|U^{(\Gamma)}_n| \to 0$ almost surely. The equations \eqref{eq_projection_sets_actor} and \eqref{eq_h2_app} imply
		\[
		\begin{aligned}
			\zeta^{(\Gamma)}_n \geq& \frac{a^{(\Gamma)}_n T}{b_n c_2} \lambda_{\min}(\Sj D_j D_j^\top) |U^{(\Gamma)}_n|^2.
		\end{aligned}
		\]
		
		Now, suppose $|U^{(\Gamma)}_n|^2 \rightarrow c$ almost surely, where \(0<c<\infty\) is a constant. Then there exists a set $Z\in \f$ with $\p(Z) > 0$ so that for every $\omega\in Z$, there exists a measurable set $Z \in \mathcal{F}$ with $\mathbb{P}(Z) = 1$ such that, for every $\omega \in Z$, there exists a constant $0<\delta(\omega)<c$ satisfying
		\[
		|U^{(\Gamma)}_n|^2 = |\Gamma_n - \Gamma_n^*|^2 \geq c-\delta(\omega) > 0 \quad \text{for sufficiently large } n.
		\]
		
		Thus, by the assumption (iii) of \eqref{eq_Gamma_assumption_app},
		\[
		\sum \zeta^{(\Gamma)}_n \geq  \frac{T}{c_2} (c-\delta(\omega)) \lambda_{\min}(\Sj D_j D_j^\top) \sum \frac{a^{(\Gamma)}_n}{b_n} =\infty.
		\]
		This is a contradiction, proving that  $\Gamma_n \xrightarrow{a.s.} \Gamma_n^*$.
		However, $\Gamma_n^*$ converges to $\bm{0}$ almost surely. Hence, $\Gamma_n \xrightarrow{a.s.} \bm{0}$ as well.
		
	\end{proof}

        \section{Proof of Theorem 3}
        \label{appendix:proof_thm3}
        \begin{proof}
		First, it follows from  \eqref{eq_projection_sets_actor} and \eqref{eq_h2_app} that
		
		\begin{equation}
        \label{eq_inner_lower_bound}
			\begin{aligned}
				&\langle U^{(\Gamma)}_n,  h^{(\Gamma)}(\Gamma_n; \bm\theta_n, \gamma_n) \rangle\\ 
				\geq& \frac{T}{2 b_n c_2} \lambda_{\min}(\Sj D_j D_j^\top) |U^{(\Gamma)}_n|^2 := \frac{\Tilde{c}}{b_n} |U^{(\Gamma)}_n|^2,
			\end{aligned}
		\end{equation}
		where $\Tilde{c} =  \frac{T}{2 c_2} \lambda_{\min}(\Sj D_j D_j^\top) > 0$.
		
		Define \(n_1 = \inf \{n \in \mathbb{N} : \Gamma_n^* \in K^{(\Gamma)}_{n+1}\}\). For \(n \geq n_1\), we follow a similar reasoning as in the proof of Theorem \ref{thm_Gamma_convergence} to get
		\begin{equation}
			\label{eq_phi_rate_temp}
			\begin{aligned}
				\E&\left[|U^{(\Gamma)}_{n+1}|^2 \Big| \g_n \right] \leq (1-\Tilde{c}\frac{a^{(\Gamma)}_n}{b_n})|U^{(\Gamma)}_n|^2 + 4\frac{(a^{(\Gamma)}_n)^2}{b_n^2}b_n^2 \\
				&\biggl( |h^{(\Gamma)}(\Gamma_n; \bm\theta_n, \gamma_n)|^2 + (1+|\frac{b_n}{2\Tilde{c}a^{(\Gamma)}_n}|)|\beta^{(\Gamma)}_n|^2 \\
				& + |\frac{b_nG_n^2}{2\Tilde{c}(a^{(\Gamma)}_n)^3}| + |\frac{G_n}{a^{(\Gamma)}_n}|^2 + \E\left[ \left|\xi^{(\Gamma)}_{n+1} -  \beta^{(\Gamma)}_n \right|^2 \Big| \g_n\right] \biggl).
			\end{aligned}
		\end{equation}
		Moreover, the assumptions in \eqref{eq_Gamma_assumption_app} imply that  $(1+|\frac{b_n}{2\Tilde{c}a^{(\Gamma)}_n}|)|\beta^{(\Gamma)}_n|^2 \leq c(\frac{b_n}{a^{(\Gamma)}_n}|\beta^{(\Gamma)}_n|^2)$ and $\frac{|G_n|^2}{(a^{(\Gamma)}_n)^2} \leq c (\frac{b_n|G_n|^2}{(a^{(\Gamma)}_n)^3})$  for some constant $c>0$.
		When $n \geq n_1$, in view of  \eqref{eq_projection_sets_critic}, \eqref{eq_projection_sets_actor}, \eqref{eq_h2n_upper_app}, and \eqref{eq_Gamma_noise_assumption}, it follows from \eqref{eq_phi_rate_temp} that
		\[ \E\left[|U^{(\Gamma)}_{n+1}|^2 \Big| \g_n \right] \leq (1-\Tilde{c} \hat{a}_n)|U^{(\Gamma)}_n|^2 + 4\hat{a}_n^2 \hat{\eta}^{(\Gamma)}_n, \]
		where
		\begin{equation}
			\begin{aligned}
				\hat{\eta}^{(\Gamma)}_n =& c b_n^2 \biggl(1 + (c^{(\bm\phi)}_n)^8 + (c^{(\Gamma)}_n)^8 + (\log b_n)^8 + \frac{b_n|\beta^{(\Gamma)}_n|^2}{a^{(\Gamma)}_n} \\
				&+ \frac{b_n|G_n|^2}{(a^{(\Gamma)}_n)^3}\biggr) \exp{\{c(c^{(\bm\phi)}_n)^4\}},
			\end{aligned}
		\end{equation}
		which is monotonically increasing because $c^{(\bm\phi)}_n$, $c^{(\Gamma)}_n$, $b_n$ are monotonically increasing and $\frac{b_n^3}{a^{(\Gamma)}_n}|\beta^{(\Gamma)}_n|^2$, $\frac{b_n^3|G_n|^2}{(a^{(\Gamma)}_n)^3}$ are non-decreasing by the assumptions. Taking expectations on both sides of the above and
		denoting $\rho_n = \E[|U^{(\Gamma)}_n|^2]$, we get
		\begin{equation}
			\label{eq_Gamma_rate_induction}
			\rho_{n+1} \leq (1-\Tilde{c} \hat{a}_n)\rho_n + 4\hat{a}_n^2 \hat{\eta}^{(\Gamma)}_n
		\end{equation}
		when $n\geq n_1$.
		
		Next, we show \(\rho_{n+1} \leq c^\prime \hat{a}_n \hat{\eta}^{(\Gamma)}_n\) for all \(n \geq 0\), where \(c' = \max\{ \frac{\rho_1}{\hat{a}_0 \hat{\eta}^{(\Gamma)}_0},  \frac{\rho_2}{\hat{a}_1 \hat{\eta}^{(\Gamma)}_1}, \cdots,  \frac{\rho_{n_1+1}}{\hat{a}_{n_1} \hat{\eta}^{(\Gamma)}_{n_1}}, \frac{4}{\Tilde{c}} \} + 1\). Indeed, it is true when $n\leq n_1$.  Assume that $\rho_{k+1} \leq c'a_k \hat{\eta}^{(\Gamma)}_n$ is true for $n_1 \leq k \leq n-1$.
		Then \eqref{eq_Gamma_rate_induction} yields
		\[
		\begin{aligned}
			\rho_{n+1} &\leq (1-\Tilde{c} \hat{a}_n)\rho_n + 4\hat{a}_n^2 \hat{\eta}^{(\Gamma)}_n\\
			&\leq (1-\Tilde{c}\hat{a}_n)c'\hat{a}_{n-1} \hat{\eta}_{1,n-1} + 4\hat{a}_n^2 \hat{\eta}^{(\Gamma)}_n \\
			&\leq (1-\Tilde{c} \hat{a}_n)c'\hat{a}_{n}(1 + w \hat{a}_n) \hat{\eta}^{(\Gamma)}_n + 4\hat{a}_n^2 \hat{\eta}^{(\Gamma)}_n \\
			&=c'\hat{a}_n\hat{\eta}^{(\Gamma)}_n + c'\hat{\eta}^{(\Gamma)}_n\hat{a}_n^2 \biggl(w - \Tilde{c} - \Tilde{c} w\hat{a}_n + \frac{4}{c'} \biggl).
		\end{aligned}
		\]
		Consider the function
		\[
		{f}(x) = c'\hat{\eta}^{(\Gamma)}_n x^2 \biggl(w - \Tilde{c} - \Tilde{c} w x + \frac{4}{c'} \biggl),
		\]
		which has two roots at $x_{1,2}=0$ and one root at $x_3=\frac{w - (\Tilde{c} - \frac{4}{c'})}{cw}$.
		Because  $\Tilde{c} - \frac{4}{c'} > 0$, we can  choose $0<w < \Tilde{c} - \frac{4}{c'}$ so that $x_3 < 0$.
		So ${f}(x) < 0$ when $x > 0$, leading to
		\[ c'\hat{\eta}^{(\Gamma)}_n\hat{a}_n^2 \biggl(w- \Tilde{c} - \Tilde{c}w\hat{a}_n + \frac{4}{c'} \biggl) < 0, \;\;\forall n
		\]
		because $\hat{a}_n>0$. We have now proved $\E [|U^{(\bm\phi)}_{n+1}|^2] \leq c' \hat{a}_n \hat{\eta}^{(\Gamma)}_n$.
		
		In particular, under the setting of Remark \ref{remark_example_gamma} and by Lemma \ref{lemma_gn_bound}, we can verify that $\frac{b_n^3}{a^{(\Gamma)}_n}|\beta^{(\Gamma)}_n|^2$ and $\frac{b_n^3|G_n|^2}{(a^{(\Gamma)}_n)^3}$ are non-decreasing sequences of $n$, and  $\hat{a}_n=\Theta(n^{-1})$. Then
		\begin{equation}
			\begin{aligned}
				\hat{\eta}^{(\Gamma)}_n\leq c n^{\frac{1}{2}} (e \vee \log n)^{p_2} (1 \vee \log \log n)^{\frac{4}{3}},
			\end{aligned}
		\end{equation}
		where $c$ and $p_2$ are positive constants. Since \(n \geq 1\), it follows that
		\[
        \begin{aligned}
		&\E[|\Gamma_{n+1} - \Gamma_{n+1}^*|^2] \leq c \frac{(e \vee \log n)^{p_2} (1 \vee \log\log                 n)^{\frac{4}{3}}}{n^{\frac{1}{2}}}\\
            \leq& c \frac{(e \vee \log (n+1))^{p_2} (1 \vee \log\log (n+1))^{\frac{4}{3}}}{(n+1)^{\frac{1}{2}}} \frac{(n+1)^{\frac{1}{2}}}{n^{\frac{1}{2}}}\\
            \leq& c\sqrt{2} \frac{(e \vee \log (n+1))^{p_2} (1 \vee \log\log (n+1))^{\frac{4}{3}}}{(n+1)^{\frac{1}{2}}}.
        \end{aligned}
		\]
		
		Finally, noting that $\Gamma_n^*$ converges to $\bm0$ with the rate of $\frac{1}{b_n}$ by \eqref{eq_phi_2n_star}, we have
		
		\[
		\begin{aligned}
			\E[|\Gamma_n|^2] =&\E[|\Gamma_n-\Gamma_n^* + \Gamma_n^*|^2]\leq \E[|\Gamma_n-\Gamma_n^* |^2] + \E[|\Gamma_n^*|^2]\\
			\leq& c' \frac{(e \vee \log n)^{p_2} (1 \vee \log\log n)^{\frac{4}{3}}}{n^{\frac{1}{2}}} + c' \frac{\alpha^{\frac{1}{2}}}{(n+\beta)^{\frac{1}{2}}}\\
			\leq& c \frac{(e \vee \log n)^{p_2} (1 \vee \log\log n)^{\frac{4}{3}}}{n^{\frac{1}{2}}}.
		\end{aligned}
		\]
		
		The proof is complete.
		
	\end{proof}

        \section{Proof of Theorem 4}
        \label{appendix:proof_thm4}

        \begin{proof}
		When \(x_0 \neq 0\), the proof is similar to that of \cite[Theorem 4.1]{huang2024sublinear} with only minor modifications needed. When \(x_0 = 0\), the proof becomes more delicate as the convergence behavior of \(\bm\phi\) now depends explicitly on the actor exploration parameter \(\Gamma\). Here, we  highlight the differences needed.
		
		Denote
		\[
		h^{(\bm\phi)}(\bm\phi_n, \Gamma_n; \bm\theta_n, \gamma_n) = \E[Y_n(T) \mid \bm\theta_n, \bm{\phi}_n, \gamma_n, \Gamma_n],
		\]
		and the noise
		\[
		\xi^{(\bm\phi)}_{n} = Y_n(T) - h^{(\bm\phi)}(\bm\phi_n, \Gamma_n; \bm\theta_n, \gamma_n).
		\]
The updating rule \eqref{eq_phi_update} for \(\bm\phi\) is now
		\begin{equation}
			\label{eq_phi_update_app}
			\bm\phi_{n+1} = \Pi_{K^{(\bm\phi)}_{n+1}} (\bm\phi_n + a^{(\bm\phi)}_n[h^{(\bm\phi)}(\bm\phi_n, \Gamma_n; \bm\theta_n. \gamma_n) + \xi^{(\bm\phi)}_{n}]).
		\end{equation}
		
		Similar to \cite[Section B]{huang2024sublinear}, we obtain
		\begin{equation}
			\label{eq_h1_app}
			\begin{aligned}
				h^{(\bm\phi)}(\bm\phi_n, \Gamma_n; \bm\theta_n, \gamma_n) = -l(\bm\phi_n, \Gamma_n; \bm\theta_n, \gamma_n)(\bm\phi_n - \bm\phi^*),
			\end{aligned}
		\end{equation}
		where
		\begin{equation}
			\label{eq_l_definition}
			l(\bm\phi_n, \Gamma_n; \bm\theta_n, \gamma_n) = (\Sj D_j D_j^\top) \int_0^T k_1(t; \bm\theta_n)\mathbb{E} [x_n(t)^2] \dd t.
		\end{equation}
		
		When \(x_0 \neq 0\), we have \(l(\bm\phi_n, \Gamma_n; \bm\theta_n, \gamma_n) \geq \bar{c} I\), where \(0 < \bar{c}' < 1\) is independent of \(\Gamma_n\). The same proof of \cite[Theorem 4.1]{huang2024sublinear} then applies.
		
		However, when \(x_0 = 0\),
		\begin{equation}
			\label{eq_h1_l_lower_bound_eq0_app}
			\begin{aligned}
				l(\bm\phi_n, \Gamma_n; \bm\theta_n, \gamma_n) &\geq (\Sj D_j D_j^\top) \int_0^T \frac{c' |\Gamma_n| t}{c_2}  \dd t \\
				&= \frac{(\Sj D_j D_j^\top) T^2 c'}{2 c_2} |\Gamma_n| \geq \bar{c}' |\Gamma_n| I.
			\end{aligned}
		\end{equation}
Thus \(l\) no longer admits a uniform lower bound strictly away from \(\bm{0}\) because \(\Gamma_n\rightarrow \bm{0}\). 
To get around this, for establishing part (a), we make use of the condition \(\sum \frac{a^{(\bm\phi)}_n}{b_n} = \infty\), which ensures that the argument  in proving \cite[Theorem B.3]{huang2024sublinear} remains valid. (Notice that the hyperparameters \(a^{(\bm\phi)}_n\) and \(b_n\) specified in Theorem \ref{thm_convergence_and_rate_gamma} also satisfy this condition.) On the other hand, for proving part (b), we reinterpret \(\hat{a}_n = \frac{a^{(\bm\phi)}_n}{b_n}\) as the effective learning rate and follow the strategy outlined in the proof of Theorem \ref{thm_Gamma_rate}. This adjustment allows the analysis in \cite[Appendix B.6]{huang2024sublinear} to be carried over to our setting.
	\end{proof}

        \section{Proof of Theorem 5}
        \label{appendix:proof_thm5}
        \begin{proof}
		Following \cite[Lemma B.7, B.8]{huang2024sublinear}, we have
		\begin{equation}
			\label{eq_j_bar}
			\Bar{J}(\bm\phi, \Gamma) = f(a(\bm\phi)) + (\Sj D_j^\top \Gamma D_j) g(a(\bm\phi)),
		\end{equation}
		where
		\begin{equation}
			\label{eq_a_phi}
			a(\bm\phi)= 2A + 2B^\top \bm\phi + \Sj (C_j^2 + 2C_jD_j^\top\bm\phi + D_j^\top\bm\phi\bm\phi^\top D_j),
		\end{equation}
		and
		\begin{equation}
			\label{eq_f_definition}
			f(a) =
			\begin{cases}
				\frac{x_0^2(-H - QT)}{2} & \text{if } a = 0, \\
				\frac{1}{2a}(Q - e^{aT}Q - He^{aT}a)x_0^2 & \text{if } a \neq 0,
			\end{cases}
		\end{equation}
		\begin{equation}
			\label{eq_g_definition}
			g(a) =
			\begin{cases}
				\frac{T(-2H - QT)}{4} & \text{if } a = 0, \\
				\frac{1}{2a^2}(QTa + Q + Ha - e^{aT}Q - He^{aT}a) & \text{if } a \neq 0.
			\end{cases}
		\end{equation}
			Clearly, both $f$ and $g$ are continuously differentiable, non-increasing, and strictly negative everywhere.
		
		Next, we present the proofs under  \(x_0 \neq 0\) and \(x_0 = 0\).
		
		\textbf{Case I: \(x_0 \neq 0.\)}		
		By \eqref{eq_j_bar}, we have
		\begin{equation}
			\label{eq_all_terms}
			\begin{aligned}
				&\Bar{J}(\bm\phi^*, 0) - \Bar{J}(\bm\phi_n, \Gamma_n) \\
				=& [f(a(\bm\phi^*)) - f(a(\bm\phi_n))]-(\Sj D_j^\top \Gamma_n D_j)g(a(\bm\phi^*))  \\
				&+ (\Sj D_j^\top \Gamma_n D_j)[g(a(\bm\phi^*)) - g(a(\bm\phi_n))].
			\end{aligned}
		\end{equation}
The goal is to establish bounds for the three terms on the right-hand side of \eqref{eq_all_terms}, adapting the arguments of the proof  of \cite[Theorem 4.2]{huang2024sublinear} to the current setting. 		
		The first term, \(f(a(\bm\phi^*)) - f(a(\bm\phi_n))\), can be bounded exactly as in the proof  of \cite[Theorem 4.2]{huang2024sublinear}, since it does not involve \(\Gamma_n\). For the second term, \(\left(\sum_j D_j^\top \Gamma_n D_j\right)g(a(\bm\phi^*))\), and the third term, \(\left(\sum_j D_j^\top \Gamma_n D_j\right)[g(a(\bm\phi^*)) - g(a(\bm\phi_n))]\), although \(\Gamma_n\) is now stochastic rather than deterministic, the bounds can still be derived by utilizing part (b) of Theorem \ref{thm_convergence_and_rate_gamma} and applying the H\"older and Cauchy--Schwarz inequalities. These modifications allow the original argument in the proof  of \cite[Theorem 4.2]{huang2024sublinear} to follow through.
		
		
		\textbf{Case II: \(x_0 = 0\).} 		
		It follows from \eqref{eq_f_definition} that $f(a(\bm\phi))=0$ for any $\bm\phi$. Hence  		\begin{equation}
			\label{eq_all_terms2}
			\begin{aligned}
				&\bar{J}(\bm\phi^*, 0) - \bar{J}(\bm\phi_n, \Gamma_n) \\
				=& -(\Sj D_j^\top \Gamma_n D_j)g(\bm\phi^*) + (\Sj D_j^\top \Gamma_n D_j)(g(\bm\phi^*) - g(\bm\phi_n)).
			\end{aligned}
		\end{equation}
		The bound for the term \(\left(\sum_j D_j^\top \Gamma_n D_j\right)g(a(\bm\phi^*))\) can be derived  exactly  as with the second term in Case I. Moreover, utilizing the MSE of \(\bm\phi_n\) for \(x_0=0\), we can also derive the bound for  the other term, \((\sum_j D_j^\top \Gamma_n D_j)(g(\bm\phi^*) - g(\bm\phi_n))\), analogously to Case I.
		
	\end{proof}

		\section{More Details for Numerical Experiments}
		\label{sec_experiment_details}
		
		To ensure full replicability, we set the random seed from 1 to 100 for each independent run in both LQRL\_Adaptive and the model-based benchmark in the first and second experiments. For the third experiment comparing LQRL\_Adaptive and LQRL\_Fixed, we set the random seed from 1 to 1,000 for each independent run in both scenarios. For the fourth experiment under the randomized environment, we set the random seed from 1 to 10,000.

		The setup for our first experiment using the model-free algorithm with data-driven exploration (LQRL\_Adaptive), Algorithm \ref{algo_rl_lq}, is as follows. The initial policy mean parameter is set to \(\bm\phi_0 = -1.1\), with the actor and critic exploration levels initialized at \(\Gamma_0 = 0.5\) and \(\gamma_0 = 2\), respectively. The learning rates are specified by \(a^{(\bm\phi)}_n = \frac{0.05}{(n+1)^{3/4}}\) and \(a^{(\Gamma)}_n = \frac{1}{(n+1)^{3/4}}\). For computational efficiency, the projection sets are defined as \([-2.25, -1.1]\) for \(\bm\phi_n\) and \([0,1]\) for \(\Gamma_n\). Although the projections were originally specified for theoretical convergence guarantees, they are slightly modified in our experiments to accelerate convergence. The projection sets for \(\bm\theta\) and \(\gamma\) remain unbounded. The deterministic sequence \(b_n\) is chosen as \(b_n = 20 \max(1, (n+1)^{1/4})\).
        \change{Additionally, we use the simplified critic specification \(k_1(t;\bm\theta)\equiv 1\) and \(k_3(t;\bm\theta,\gamma)\equiv0\), which trivially satisfies the boundedness conditions in \eqref{eq_k_bounds}. Under this specialization, the temperature update reduces to \(\gamma_n=c_\gamma/b_n\). Thus, the numerical implementation corresponds to a special case of the general adaptive critic update \eqref{eq_gamma_definition}. This simplification is used to avoid introducing  a critic neural-network so as to save computational cost. In this sense, the experiments provide a conservative implementation of the proposed framework, showing that strong empirical performance can already be obtained under a simple bounded critic specification. Note that this does not contradict  the general theoretical framework, since the convergence and regret analysis only requires the bounds in \eqref{eq_k_bounds} and does not rely on the explicit functional forms of \(k_1\) and \(k_3\). More general choices, including neural-network parameterizations satisfying \eqref{eq_k_bounds}, can be incorporated directly into Algorithm~\ref{algo_rl_lq} through the policy-evaluation update  \eqref{eq_theta_update_discrete}.}
		
		The model-based benchmark is adapted from \cite{szpruch2024optimal} and extended by \cite[Algorithm A.1]{huang2024sublinear} to accommodate state- and control-dependent volatility in our setting. The initial parameter estimations are set to be \(A = B = C = D = 10\), resulting in the same initial \(\bm\phi_0 = -1.1\) as in LQRL\_Adaptive. The initial exploration level \(\Gamma_0 = 0.5\) is also matched. The deterministic exploration schedule follows \(\Gamma_n = 0.5 (n+1)^{-1}\), consistent with \cite{huang2024sublinear}. Finally, the projection set for \(\bm\phi_n\) remains \([-2.25, -1.1]\), aligning with the setup used in LQRL\_Adaptive.

		For all the plots in the first and second experiments, for both our adaptive model-free algorithm and model-based benchmarks, regression lines are fitted using data from iterations 5,000 to 100,000. This avoids the influence of early-stage learning noise. In all the regret plots, we use the median and exclude extreme values; similar conclusions hold when using the mean instead.
		
		In our third experiment, comparing LQRL\_Adaptive and LQRL\_Fixed, most settings remain identical to the first experiment, except that the bounds \(c^{(\bm\phi)}_n\) and \(c^{(\Gamma)}_n\) are set to be 20 to accommodate the scenarios under consideration. For the fourth experiment, to account for potentially large values of \(\bm\phi\) and \(\Gamma\) due to random model parameters, we further increase the bounds \(c^{(\bm\phi)}_n\) and \(c^{(\Gamma)}_n\) to be 100.
		
	\end{appendices}

	

	\bibliographystyle{ieeetr}
	\bibliography{ref}

\begin{thebibliography}{10}

\bibitem{anderson2007optimal}
B.~D. Anderson and J.~B. Moore, {\em Optimal Control: Linear Quadratic
  Methods}.
\newblock Courier Corporation, 2007.

\bibitem{YZbook}
J.~Yong and X.~Y. Zhou, {\em Stochastic Controls: Hamiltonian Systems and HJB
  Equations}.
\newblock New York, NY: Spinger, 1999.

\bibitem{chen1998stochastic}
S.~Chen, X.~Li, and X.~Y. Zhou, ``Stochastic linear quadratic regulators with
  indefinite control weight costs,'' {\em SIAM Journal on Control and
  Optimization}, vol.~36, no.~5, pp.~1685--1702, 1998.

\bibitem{ait2000well}
M.~Ait~Rami, X.~Y. Zhou, and J.~Moore, ``Well-posedness and attainability of
  indefinite stochastic linear quadratic control in infinite time horizon,''
  {\em Systems \& Control Letters}, vol.~41, no.~2, pp.~123--133, 2000.

\bibitem{rami2000linear}
M.~A. Rami and X.~Y. Zhou, ``Linear matrix inequalities, riccati equations, and
  indefinite stochastic linear quadratic controls,'' {\em IEEE Transactions on
  Automatic Control}, vol.~45, no.~6, pp.~1131--1143, 2000.

\bibitem{rami2001solvability}
M.~A. Rami, X.~Chen, J.~B. Moore, and X.~Y. Zhou, ``Solvability and asymptotic
  behavior of generalized riccati equations arising in indefinite stochastic lq
  controls,'' {\em IEEE Transactions on Automatic Control}, vol.~46, no.~3,
  pp.~428--440, 2001.

\bibitem{yao2001primal}
D.~D. Yao, S.~Zhang, and X.~Y. Zhou, ``A primal-dual semi-definite programming
  approach to linear quadratic control,'' {\em IEEE Transactions on Automatic
  Control}, vol.~46, no.~9, pp.~1442--1447, 2001.

\bibitem{li2022stabilization}
H.~Li, Q.~Qi, and H.~Zhang, ``Stabilization control for it{\^o} stochastic
  system with indefinite state and control weight costs,'' {\em International
  Journal of Control}, vol.~95, no.~2, pp.~295--302, 2022.

\bibitem{gashi2023optimal}
B.~Gashi and H.~Hua, ``Optimal regulators for a class of nonlinear stochastic
  systems,'' {\em International Journal of Control}, vol.~96, no.~1,
  pp.~136--146, 2023.

\bibitem{wang2024indefinite}
G.~Wang and W.~Wang, ``Indefinite linear-quadratic optimal control of
  mean-field stochastic differential equation with jump diffusion: an
  equivalent cost functional method,'' {\em IEEE Transactions on Automatic
  Control}, 2024.

\bibitem{wu2025stochastic}
F.~Wu, X.~Li, and X.~Zhang, ``Stochastic linear quadratic optimal control
  problems with regime-switching jumps in infinite horizon,'' {\em SIAM Journal
  on Control and Optimization}, vol.~63, no.~2, pp.~852--891, 2025.

\bibitem{merton1980estimating}
R.~C. Merton, ``On estimating the expected return on the market: An exploratory
  investigation,'' {\em Journal of Financial Economics}, vol.~8, no.~4,
  pp.~323--361, 1980.

\bibitem{luenberger1998investment}
D.~G. Luenberger, {\em {Investment Science}}.
\newblock Oxford University Press, 1998.

\bibitem{rustem2009algorithms}
B.~Rustem and M.~Howe, {\em Algorithms for worst-case design and applications
  to risk management}.
\newblock Princeton University Press, 2009.

\bibitem{bradtke1992reinforcement}
S.~Bradtke, ``Reinforcement learning applied to linear quadratic regulation,''
  {\em Advances in Neural Information Processing Systems}, vol.~5, 1992.

\bibitem{fazel2018global}
M.~Fazel, R.~Ge, S.~Kakade, and M.~Mesbahi, ``Global convergence of policy
  gradient methods for the linear quadratic regulator,'' in {\em International
  Conference on Machine Learning}, pp.~1467--1476, PMLR, 2018.

\bibitem{abbasi2011regret}
Y.~Abbasi-Yadkori and C.~Szepesv{\'a}ri, ``Regret bounds for the adaptive
  control of linear quadratic systems,'' in {\em Proceedings of the 24th Annual
  Conference on Learning Theory}, pp.~1--26, JMLR Workshop and Conference
  Proceedings, 2011.

\bibitem{abeille2018improved}
M.~Abeille and A.~Lazaric, ``Improved regret bounds for {T}hompson sampling in
  linear quadratic control problems,'' in {\em International Conference on
  Machine Learning}, pp.~1--9, PMLR, 2018.

\bibitem{hambly2021policy}
B.~Hambly, R.~Xu, and H.~Yang, ``Policy gradient methods for the noisy linear
  quadratic regulator over a finite horizon,'' {\em SIAM Journal on Control and
  Optimization}, vol.~59, no.~5, pp.~3359--3391, 2021.

\bibitem{wang2021global}
W.~Wang, J.~Han, Z.~Yang, and Z.~Wang, ``Global convergence of policy gradient
  for linear-quadratic mean-field control/game in continuous time,'' in {\em
  International Conference on Machine Learning}, pp.~10772--10782, PMLR, 2021.

\bibitem{zhou2023single}
M.~Zhou and J.~Lu, ``Single timescale actor-critic method to solve the linear
  quadratic regulator with convergence guarantees,'' {\em Journal of Machine
  Learning Research}, vol.~24, no.~222, pp.~1--34, 2023.

\bibitem{sutton2011reinforcement}
R.~S. Sutton and A.~G. Barto, {\em Reinforcement Learning: An Introduction}.
\newblock Cambridge, MA: MIT Press, 2018.

\bibitem{aubret2019survey}
A.~Aubret, L.~Matignon, and S.~Hassas, ``A survey on intrinsic motivation in
  reinforcement learning,'' {\em arXiv preprint arXiv:1908.06976}, 2019.

\bibitem{schmidhuber2010formal}
J.~Schmidhuber, ``Formal theory of creativity, fun, and intrinsic motivation
  (1990--2010),'' {\em IEEE Transactions on Autonomous Mental Development},
  vol.~2, no.~3, pp.~230--247, 2010.

\bibitem{achiam2017surprise}
J.~Achiam and S.~Sastry, ``Surprise-based intrinsic motivation for deep
  reinforcement learning,'' {\em arXiv preprint arXiv:1703.01732}, 2017.

\bibitem{pathak2017curiosity}
D.~Pathak, P.~Agrawal, A.~A. Efros, and T.~Darrell, ``Curiosity-driven
  exploration by self-supervised prediction,'' in {\em International Conference
  on Machine Learning}, pp.~2778--2787, PMLR, 2017.

\bibitem{burda2018large}
Y.~Burda, H.~Edwards, D.~Pathak, A.~Storkey, T.~Darrell, and A.~A. Efros,
  ``Large-scale study of curiosity-driven learning,'' {\em arXiv preprint
  arXiv:1808.04355}, 2018.

\bibitem{burda2018exploration}
Y.~Burda, H.~Edwards, A.~Storkey, and O.~Klimov, ``Exploration by random
  network distillation,'' {\em arXiv preprint arXiv:1810.12894}, 2018.

\bibitem{osband2018randomized}
I.~Osband, J.~Aslanides, and A.~Cassirer, ``Randomized prior functions for deep
  reinforcement learning,'' {\em Advances in neural information processing
  systems}, vol.~31, 2018.

\bibitem{menard2021fast}
P.~M{\'e}nard, O.~D. Domingues, A.~Jonsson, E.~Kaufmann, E.~Leurent, and
  M.~Valko, ``Fast active learning for pure exploration in reinforcement
  learning,'' in {\em International Conference on Machine Learning},
  pp.~7599--7608, PMLR, 2021.

\bibitem{tang2017exploration}
H.~Tang, R.~Houthooft, D.~Foote, A.~Stooke, O.~Xi~Chen, Y.~Duan, J.~Schulman,
  F.~DeTurck, and P.~Abbeel, ``\# exploration: A study of count-based
  exploration for deep reinforcement learning,'' {\em Advances in neural
  information processing systems}, vol.~30, 2017.

\bibitem{ecoffet2019go}
A.~Ecoffet, J.~Huizinga, J.~Lehman, K.~O. Stanley, and J.~Clune, ``Go-explore:
  A new approach for hard-exploration problems,'' {\em arXiv preprint
  arXiv:1901.10995}, 2019.

\bibitem{ecoffet2021first}
A.~Ecoffet, J.~Huizinga, J.~Lehman, K.~O. Stanley, and J.~Clune, ``First
  return, then explore,'' {\em Nature}, vol.~590, no.~7847, pp.~580--586, 2021.

\bibitem{vecerik2017leveraging}
M.~Vecerik, T.~Hester, J.~Scholz, F.~Wang, O.~Pietquin, B.~Piot, N.~Heess,
  T.~Roth{\"o}rl, T.~Lampe, and M.~Riedmiller, ``Leveraging demonstrations for
  deep reinforcement learning on robotics problems with sparse rewards,'' {\em
  arXiv preprint arXiv:1707.08817}, 2017.

\bibitem{garcia2015comprehensive}
J.~Garc{\i}a and F.~Fern{\'a}ndez, ``A comprehensive survey on safe
  reinforcement learning,'' {\em Journal of Machine Learning Research},
  vol.~16, no.~1, pp.~1437--1480, 2015.

\bibitem{saunders2017trial}
W.~Saunders, G.~Sastry, A.~Stuhlmueller, and O.~Evans, ``Trial without error:
  Towards safe reinforcement learning via human intervention,'' {\em arXiv
  preprint arXiv:1707.05173}, 2017.

\bibitem{munos2006policy}
R.~Munos, ``Policy gradient in continuous time,'' {\em Journal of Machine
  Learning Research}, vol.~7, pp.~771--791, 2006.

\bibitem{tallec2019making}
C.~Tallec, L.~Blier, and Y.~Ollivier, ``Making deep {Q}-learning methods robust
  to time discretization,'' in {\em International Conference on Machine
  Learning}, pp.~6096--6104, PMLR, 2019.

\bibitem{park2021time}
S.~Park, J.~Kim, and G.~Kim, ``Time discretization-invariant safe action
  repetition for policy gradient methods,'' {\em Advances in Neural Information
  Processing Systems}, vol.~34, pp.~267--279, 2021.

\bibitem{huang2024sublinear}
Y.~Huang, Y.~Jia, and X.~Y. Zhou, ``Sublinear regret for a class of
  continuous-time linear-quadratic reinforcement learning problems,'' {\em SIAM
  Journal on Control and Optimization}, vol.~63, no.~5, pp.~3452--3474, 2025.

\bibitem{szpruch2024optimal}
L.~Szpruch, T.~Treetanthiploet, and Y.~Zhang, ``Optimal scheduling of entropy
  regularizer for continuous-time linear-quadratic reinforcement learning,''
  {\em SIAM Journal on Control and Optimization}, vol.~62, no.~1, pp.~135--166,
  2024.

\bibitem{hu2003indefinite}
Y.~Hu and X.~Y. Zhou, ``Indefinite stochastic riccati equations,'' {\em SIAM
  Journal on Control and Optimization}, vol.~42, no.~1, pp.~123--137, 2003.

\bibitem{qian2013existence}
Z.~Qian and X.~Y. Zhou, ``Existence of solutions to a class of indefinite
  stochastic riccati equations,'' {\em SIAM journal on Control and
  Optimization}, vol.~51, no.~1, pp.~221--229, 2013.

\bibitem{wang2020reinforcement}
H.~Wang, T.~Zariphopoulou, and X.~Y. Zhou, ``Reinforcement learning in
  continuous time and space: A stochastic control approach,'' {\em Journal of
  Machine Learning Research}, vol.~21, no.~198, pp.~1--34, 2020.

\bibitem{szpruch2021exploration}
L.~Szpruch, T.~Treetanthiploet, and Y.~Zhang, ``Exploration-exploitation
  trade-off for continuous-time episodic reinforcement learning with
  linear-convex models,'' {\em arXiv preprint arXiv:2112.10264}, 2021.

\bibitem{jia2022policy}
Y.~Jia and X.~Y. Zhou, ``Policy gradient and actor-critic learning in
  continuous time and space: Theory and algorithms,'' {\em Journal of Machine
  Learning Research}, vol.~23, no.~275, pp.~1--50, 2022.

\bibitem{jia2021policy}
Y.~Jia and X.~Y. Zhou, ``Policy evaluation and temporal-difference learning in
  continuous time and space: A martingale approach,'' {\em Journal of Machine
  Learning Research}, vol.~23, no.~154, pp.~1--55, 2022.

\bibitem{robbins1951stochastic}
H.~Robbins and S.~Monro, ``A stochastic approximation method,'' {\em The Annals
  of Mathematical Statistics}, pp.~400--407, 1951.

\bibitem{lai2003stochastic}
T.~L. Lai, ``Stochastic approximation,'' {\em The Annals of Statistics},
  vol.~31, no.~2, pp.~391--406, 2003.

\bibitem{borkar2009stochastic}
V.~S. Borkar, {\em Stochastic approximation: A dynamical systems viewpoint},
  vol.~48.
\newblock Springer, 2009.

\bibitem{chau2014overview}
M.~Chau and M.~C. Fu, ``An overview of stochastic approximation,'' {\em
  Handbook of Simulation Optimization}, pp.~149--178, 2014.

\bibitem{andradottir1995stochastic}
S.~Andrad{\'o}ttir, ``A stochastic approximation algorithm with varying
  bounds,'' {\em Operations Research}, vol.~43, no.~6, pp.~1037--1048, 1995.

\bibitem{huang2024mean}
Y.~Huang, Y.~Jia, and X.~Y. Zhou, ``Mean--variance portfolio selection by
  continuous-time reinforcement learning: Algorithms, regret analysis, and
  empirical study,'' {\em arXiv preprint arXiv:2412.16175}, 2024.

\bibitem{jia2023q}
Y.~Jia and X.~Y. Zhou, ``q-learning in continuous time,'' {\em Journal of
  Machine Learning Research}, vol.~24, no.~161, pp.~1--61, 2023.

\bibitem{broadie2011general}
M.~Broadie, D.~Cicek, and A.~Zeevi, ``{General bounds and finite-time
  improvement for the Kiefer-Wolfowitz stochastic approximation algorithm},''
  {\em Operations Research}, vol.~59, no.~5, pp.~1211--1224, 2011.

\bibitem{kloedennumerical}
P.~E. Kloeden and E.~Platen, {\em Numerical Solution of Stochastic Differential
  Equations}, vol.~23.
\newblock Springer, 1992.

\bibitem{shamir2013stochastic}
O.~Shamir and T.~Zhang, ``Stochastic gradient descent for non-smooth
  optimization: Convergence results and optimal averaging schemes,'' in {\em
  International conference on machine learning}, pp.~71--79, PMLR, 2013.

\bibitem{nie2022explore}
G.~Nie, M.~Agarwal, A.~K. Umrawal, V.~Aggarwal, and C.~J. Quinn, ``An
  explore-then-commit algorithm for submodular maximization under full-bandit
  feedback,'' in {\em Uncertainty in Artificial Intelligence}, pp.~1541--1551,
  PMLR, 2022.

\bibitem{tang2024regret}
W.~Tang and X.~Y. Zhou, ``Regret of exploratory policy improvement and $ q
  $-learning,'' {\em arXiv preprint arXiv:2411.01302}, 2024.

\bibitem{robbins1971convergence}
H.~Robbins and D.~Siegmund, ``A convergence theorem for non negative almost
  supermartingales and some applications,'' in {\em Optimizing Methods in
  Statistics}, pp.~233--257, Elsevier, 1971.

\end{thebibliography}

    \subsection*{Yilie Huang}
\noindent \includegraphics[width=1in,height=1.25in,clip,keepaspectratio]{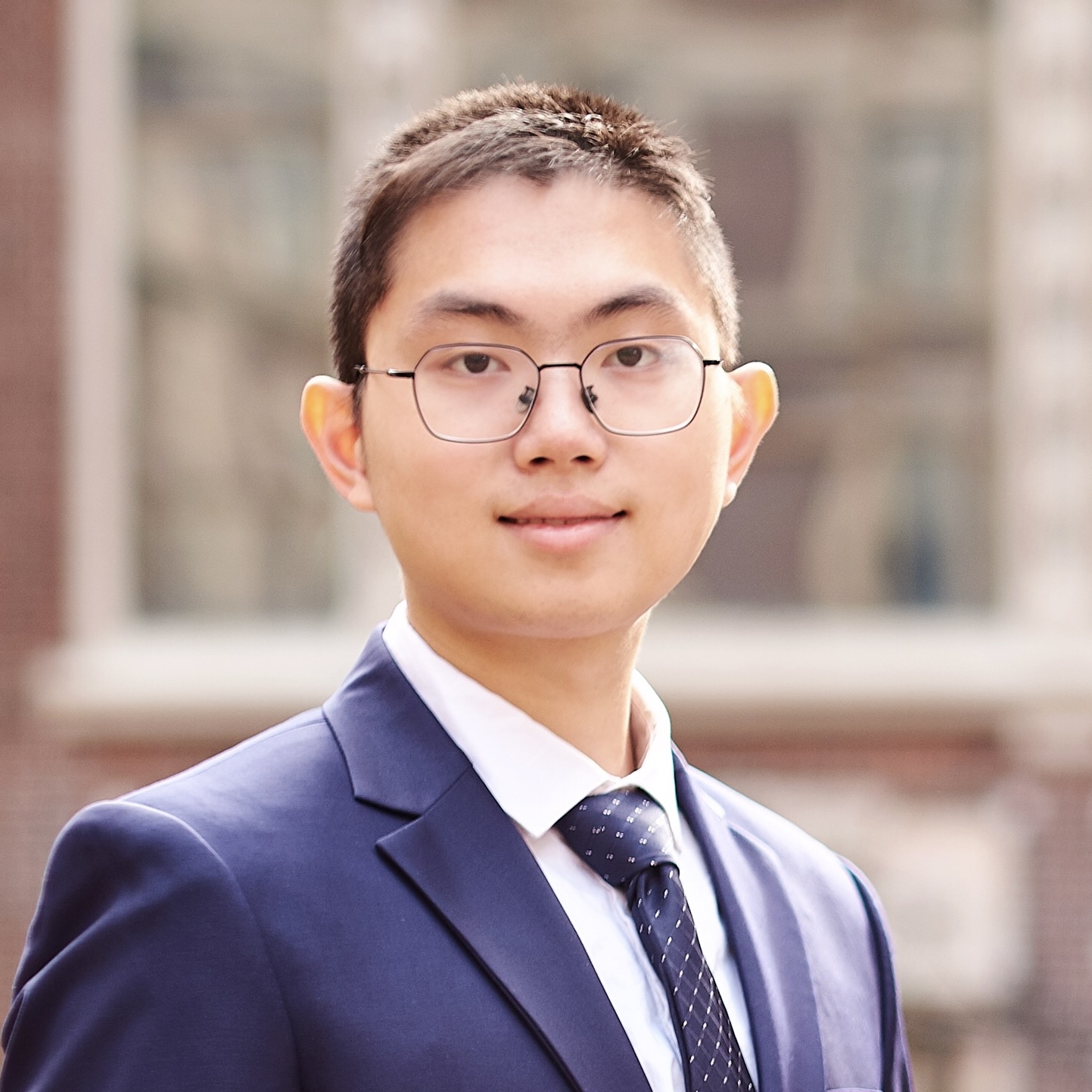}\par\smallskip\noindent (Member, IEEE) was born in Guangzhou, China. He received the B.S. degree in mathematics and applied mathematics from Zhejiang University, Hangzhou, China, in 2017, and the M.S. and Ph.D. degrees in operations research from Columbia University, New York, USA, in 2018 and 2024, respectively.

    He is currently an Assistant Professor with the Department of Applied Mathematics, The Hong Kong Polytechnic University, Hong Kong. Prior to this appointment, he was a Postdoctoral Research Scientist with the Department of Industrial Engineering and Operations Research, Columbia University, New York, USA.
    His research lies at the intersection of stochastic control, reinforcement learning, diffusion models for generative AI, and financial engineering, with a focus on the development, analysis, and empirical evaluation of continuous-time reinforcement learning algorithms for decision making, control, and financial systems under uncertainty.

    Dr. Huang is a Member of the Institute for Operations Research and the Management Sciences (INFORMS), the Society for Industrial and Applied Mathematics (SIAM), the Association for Computing Machinery (ACM), and is a Chartered Financial Analyst (CFA) charterholder.

    \subsection*{Xun Yu Zhou}
\noindent \includegraphics[width=1in,height=1.25in,clip,keepaspectratio]{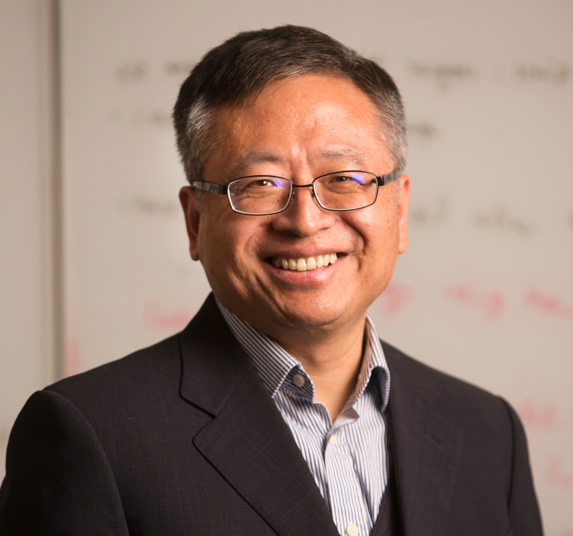}\par\smallskip\noindent (Fellow, IEEE) was born in Suzhou, China. He received his B.Sc. in mathematics in 1984, and his Ph.D. in operations research and control theory in 1989, both from Fudan University, Shanghai, China.

     He is currently  the Liu Family Professor of Industrial Engineering and Operations Research and Director of the Nie Center for Intelligent Asset Management at Columbia University in New York. He was the Nomura Professor of Mathematical Finance, the Director of Nomura Center for Mathematical Finance and the Director of Oxford-Nie Financial Big Data Lab at the University of Oxford during 2007-2016 before joining Columbia.

    He is known for his work in indefinite stochastic LQ control theory and application to dynamic mean--variance portfolio selection, in asset allocation and pricing under prospect theory, and in general time-inconsistent problems. His current research focuses on reinforcement learning for controlled diffusion processes and applications to GenAI and intelligent wealth management solutions. He has addressed the 2010 International Congress of Mathematicians, and has been awarded the Wolfson Research Award from The Royal Society (UK), the Outstanding Paper Prize from SIAM, the Humboldt Distinguished Lecturer, the Alexander von Humboldt Research Fellowship, the Archimedes Lecturer at Columbia, and Distinguished Faculty Teaching Award at Columbia University. \\

    Professor Zhou is a Fellow of SIAM, a member of INFORMS, and a Life Member of the Bachelier Finance Society. He is or was on the editorial boards of numerous journals including {\it IEEE Transactions on Automatic Control} and {\it SIAM Journal on Control and Optimization}.

\end{document}